\documentclass[11pt,letter]{article}
\pdfoutput=1
\usepackage{fullpage}
\usepackage{amsmath,amsbsy,amsfonts,amssymb,amsthm,dsfont,units}
\usepackage{graphicx}
\usepackage{authblk}
\usepackage{mathtools}
\usepackage[usenames,dvipsnames]{color}
\usepackage{cases}
\usepackage{subfigure}
\usepackage[round]{natbib}
\usepackage[breaklinks]{hyperref}
\usepackage{mathtools}
\usepackage{algorithm}
\usepackage{algorithmic}
\usepackage{enumitem}
\usepackage{rfmacros}
\usepackage{fgksmacros}

\newif\ifcolorfuldiscussion
\colorfuldiscussionfalse

\ifcolorfuldiscussion
\newcommand{\sk}[1]{\noindent{\textcolor{magenta}{{\bf SK:} \em #1}}}
\newcommand{\rg}[1]{\noindent{\textcolor{red}{{\bf RG:} \em #1}}}
\newcommand{\sidford}[1]{\noindent{\textcolor{green}{{\bf AS:} \em #1}}}
\definecolor{electricpurple}{rgb}{0.75, 0.0, 1.0}
\newcommand{\rf}[1]{\noindent{\textcolor{electricpurple}{{\bf RF:} \em #1}}}
\else
\newcommand{\sk}[1]{}
\newcommand{\rg}[1]{}
\newcommand{\sidford}[1]{}
\newcommand{\rf}[1]{}
\fi

\title{Competing with the Empirical Risk Minimizer in a Single Pass}

\newcommand\email[1]{\texttt{\small {#1}}}

\date{}
\author[1]{Roy Frostig} 
\author[2]{Rong Ge}
\author[2]{Sham M. Kakade}
\author[3]{Aaron Sidford} 
\affil[1]{Stanford University\authorcr
\email{rf@cs.stanford.edu}\vspace{0.4em}}
\affil[2]{Microsoft Research, New England\authorcr
\email{rongge@microsoft.com}, \email{skakade@microsoft.com}\vspace{0.4em}}
\affil[3]{MIT\authorcr
\email{sidford@mit.edu}}

\begin{document}

\maketitle

\begin{abstract}
In many estimation problems, \eg linear and logistic regression, we
wish to minimize an unknown objective given only unbiased samples
of the objective function. Furthermore, we aim to achieve this using as few samples as
possible.  In the absence of computational constraints, the
minimizer of a sample average of observed data -- commonly referred
to as either the empirical risk minimizer (ERM) or the $M$-estimator
-- is widely regarded as the estimation strategy of choice due to
its desirable statistical convergence properties. Our goal in this work is to perform
as well as the ERM, on \emph{every} problem,
while minimizing the use of computational resources such as running
time and space usage.

We provide a simple streaming algorithm which, under
standard regularity assumptions on the underlying problem, enjoys
the following properties:
\begin{enumerate}[noitemsep]
\item The algorithm can be implemented in linear time with a single
pass of the observed data, using space linear in the size of a
single sample.
\item The algorithm achieves the same statistical rate of
convergence as the empirical risk minimizer on every problem, even
considering constant factors.
\item The algorithm's performance depends on the initial error at a
rate that decreases super-polynomially.
\item The algorithm is easily parallelizable.
\end{enumerate}
Moreover, we quantify the (finite-sample) rate at which the
algorithm becomes competitive with the ERM.
\end{abstract}

\section{Introduction}

Consider the following optimization problem:
\begin{align}\label{eq:P}
\min_{w \in \mathcal{S}} P(w), \quad \text{where} \quad P(w) \defeq
\E_{\psi\sim \D} [\psi(w)]
\end{align}
and $\mD$ is a distribution over convex
functions from
a Euclidean space $\mathcal{S}$ to $\R$ (\eg
$\mathcal{S}=\R^d$ in the finite dimensional setting).
Let $w_*$ be a minimizer of $P$ and suppose we observe the functions $\psi_1,\psi_2,\ldots, \psi_N$
independently sampled from $\mD$.  Our objective is to compute an estimator
$\widehat w_N$ so that the expected \emph{error} (or, equivalently,
the \emph{excess risk}):
\[
\E[P(\widehat w_N) - P(w_*)]
\]
is small, where the expectation is over the estimator $\widehat{w}_N$
(which depends on the sampled functions).

Stochastic approximation algorithms, such as stochastic gradient
descent (SGD) \citep{robbins1951stoch}, are the most widely used in practice, due
to their ease of implementation and their efficiency with regards to
runtime and memory.  Without consideration for computational
constraints, we often wish to compute the \emph{empirical risk minimizer}
(ERM; or, equivalently, the \emph{$M$-estimator}):
\begin{align}\label{eq:ERM_definition}
\widehat w^{\textrm{ERM}}_{N}
&\in \argmin_{w\in\mathcal{S}} \frac{1}{N} \sum_{i=1}^N \psi_i(w).
\end{align}
In the context of statistical modeling, the ERM is the maximum
likelihood estimator (MLE). Under certain regularity conditions, and
under correct model specification,\footnote{A well specified
statistical model is one where the data is generated under some
model in the parametric class. See the linear regression
Section~\ref{subsec:regression}.} the MLE is asymptotically
efficient, in that no unbiased estimator can have a lower variance in
the limit (see
\citet{lehmann1998theory,vandervaart2000asymptotic}).\footnote{However,
note that biased estimators, such as the James-Stein estimator, can
outperform the MLE \citep{lehmann1998theory}. } Analogous arguments
have been made in the stochastic approximation setting, where we do
not necessarily have a statistical model of the distribution $\D$
(see~\citet{kushner2003stochastic}).

The question we aim to address is as follows.  Consider the ratio:
\begin{align} \label{eq:competition-ratio}
\frac {\E[P(\widehat w^{\textrm{ERM}}_{N}) - P(w_*)]}{\E[P(\widehat
w_N) - P(w_*)]} \, .
\end{align}
We seek an algorithm to compute $\widehat w_N$ in which: (1) under
sufficient regularity conditions, this ratio approaches $1$ on
\emph{every} problem $\D$ and (2) it does so quickly, at a rate
quantifiable in terms of the number of samples, the dependence on the
initial error (and other relevant quantities), and the computational time and space usage.

\subsection{This work}

Under certain smoothness assumptions on $\psi$ and strong convexity
assumptions on $P$ (applicable to linear and logistic regression,
generalized linear models, smoothed Huber losses, and various other
$M$-estimation problems), we provide an algorithm where:
\begin{enumerate}[noitemsep]
\item The algorithm achieves the same statistical rate of convergence
as the ERM on every problem, even considering constant factors, and
we quantify the sample size at which this occurs.
\item The algorithm can be implemented in linear time with a single
pass of the observed data, using space linear in the size of a
single sample.
\item The algorithm decreases the standard notion of initial error at
a super-polynomial rate.\footnote{A function is super-polynomial if
grows faster than any polynomial.}
\item The algorithm is trivially parallelizable (see Remark~\ref{rem:parallel}).
\end{enumerate}
Table~\ref{tab:rates} compares previous (and concurrent) algorithms
that enjoy the first two guarantees; this work is the first with a
finite-sample analysis handling the more general class of problems.
Our algorithm is a variant of the stochastic variance reduced gradient
procedure of \citet{johnson2013svrg}.

Importantly, we quantify how fast we obtain a rate comparable to that
of the ERM. For the case of linear regression, we have non-trivial
guarantees when the sample size $N$ is larger than a constant times
what can be interpreted as a condition number, $\conditionNumber =
L/\mu$, where $\mu$ is a strong convexity parameter of $P$ and where
$L$ is a smoothness parameter of each $\psi$. Critically, after $N$ is larger
than $\kappa$,
the initial error is divided by a factor that can be larger than any polynomial in $N/\kappa$.

Finally, in order to address this question on a per-problem basis, we
provide both upper and lower bounds for the rate of convergence of the
ERM.

\subsection{Related work}

\definecolor{royalpurple}{rgb}{0.3, 0.1, 0.8}
\newcommand\thiswork{\noalign{\vskip-2pt}\hspace{-0.3em}\scriptsize{\textbf{\textcolor{royalpurple}{This work:}}}}
\newcommand\minicell[2]{
\begin{minipage}[t]{#1}
{#2}
\end{minipage}}

\begin{table}[t]

\centering
\begin{tabular}{|p{4.1cm}|p{1.5cm}|p{2.4cm}|p{2.1cm}|p{1.3cm}|>{\centering\arraybackslash}p{2.4cm}|}
\hline
\centering Algorithm / analysis &
\centering Problem &
\centering Step size &
\centering Initial error dependence &
\centering Parallel-izable &
Finite-sample analysis \\ \hline \hline
{\footnotesize \citet{polyak1992stoch}} &
general &
\minicell{3cm}{decaying: $1/n^{c}$} &
\centering ? &
\centering ? &
\xmark \\ \hline
{\footnotesize \citet{polyak1992stoch}  / \citet{dieuleveut2014non}} &
\minicell{1.7cm}{linear \\ regression} & 
constant &
\centering $\Omega(1/n^2)$ &
\centering ? &
\cmark \\ \hline
\thiswork & & & & & \\[-0.2em]
\minicell{2.6cm}{{\footnotesize Streaming SVRG}} &
general &
constant &
\centering $1/n^{\omega(1)}$ &
\centering \cmark &
\cmark \padRowBot \\
\hline
\end{tabular}
\caption{Comparison of known \emph{streaming} algorithms which achieve
a \emph{constant} competitive ratio to the ERM.
\citet{polyak1992stoch} is an SGD algorithm with iterate
averaging. Concurrent to and independently from our work,
\citet{dieuleveut2014non} provide a finite-sample analysis for SGD with
averaging in the linear regression problem setting (where the
learning rate can be taken as constant).
In the ``problem'' column, ``general'' indicates problems under the
regularity assumptions herein.
\citet{polyak1992stoch} require the step size to decay with
the sample size $n$, as $1/n^c$ with $c$ strictly in the range $1/2 < c < 1$.
The dependence on $c$ in a finite-sample
analysis is unclear (and tuning the decay of learning
rates is often undesirable in practice).
The initial error is $P(w_0)-P(w_*)$, where $w_0$ is the starting
point of the algorithm. We seek algorithms in which the initial error
dependence is significantly lower in order, and we write $1/n^{\omega(1)}$
to indicate that it can be driven down to an arbitrarily low-order
polynomial.
See Remark~\ref{rem:parallel} with regard to parallelization.
}
\label{tab:rates}
\end{table}

Stochastic optimization dates back to the work of
\citet{robbins1951stoch} and has seen much subsequent
work~\citep{KushnerClark,kushner2003stochastic,
nemirovski1983problem}. More recently, questions of how to quantify
and compare rates of estimation procedures~-- with implications to
machine learning problems in the streaming and large dataset
settings~-- have been raised and discussed several times (see
\citet{bottou2008tradeoffs,agarwal2014lower}).

\paragraph{Stochastic approximation.}  The pioneering work
of~\citet{polyak1992stoch} and~\citet{RuppertSlow} provides an
asymptotically optimal streaming algorithm, by averaging the
iterates of an SGD procedure.
It is unclear how quickly these algorithms converge to the rate of
the ERM in finite sample; the relevant dependencies, such as the
dependence on the initial error~-- that is, $P(w_0)-P(w_*)$ where
$w_0$ is the starting point of the algorithm~-- are not specified.  In
particular, they characterize the limiting distribution of $\sqrt{N}
(\widehat w_N - w_*)$, essentially arguing that the variance of the
iterate-averaging procedure matches the asymptotic distribution of the
ERM (see~\citet{kushner2003stochastic}).

In a series of papers,
\citet{bach2011nonasymptotic}, \citet{bach2013nonstrong}, \citet{dieuleveut2014non}, and \citet{bach2014constant}
provide non-asymptotic analysis of the same averaging schemes. Of
these, for the specific case of linear least-squares regression,
\citet{dieuleveut2014non} and \citet{bach2014constant} provide rates which are
competitive with the ERM,
concurrently and independent of results presented herein. The work in
\citet{bach2011nonasymptotic} and \citet{bach2013nonstrong} either does not
achieve the ERM rate or has a dependence on the initial error which is
not lower in order; it is rather in \citet{dieuleveut2014non} and \citet{bach2014constant} that
dependence on the initial error decaying as
$1/N^2$ is shown.

For the special case of least squares, one could adapt the
algorithm and guarantees of \citet{dieuleveut2014non,bach2014constant}, by replacing global averaging with random
restarts, to obtain super-polynomial rates (results comparable to ours
when specializing to linear regression). For more general problems, it
is unclear how such an adaptation would work~-- using constant step
sizes alone may not suffice. In contrast, as shown in Table~\ref{tab:rates}, our algorithm is identical for a
wide variety of cases and does not need decaying rates (whose choices
may be difficult in practice).

We should also note that much work has characterized rates of
convergence under various assumptions on $P$ and $\psi$ different than
our own.  Our case of interest is when $P$ is strongly convex.
For such $P$, the rates of convergence of many algorithms are
$O(1/N)$,
often achieved by averaging the iterates in some way
\citep{nemirovski2009robust,juditsky:hal-00508933,RakShaSri12,JMLR:v15:hazan14a}.
These results do not achieve a constant competitive ratio, for a
variety of reasons (they have a leading order dependencies on various
quantities, including the initial error along with strong convexity and
smoothness parameters).  Solely in terms of the dependence on the
sample size $N$, these rates are known to be
optimal~\citep{nemirovski1983problem, nesterov2004introductory,
agarwal2012lowerbounds}.

\paragraph{Empirical risk minimization ($M$-estimation).}
In statistics, it is classically argued that the MLE, under certain
restrictions, is an asymptotically efficient estimator for
well-specified statistical models
\citep{lehmann1998theory,vandervaart2000asymptotic}. Analogously, in
an optimization context, applicable to mis-specified models, similar
asymptotic arguments have been made: under certain restrictions, the
asymptotically optimal estimator is one which has a limiting variance
that is equivalent to that of the ERM
\citep{anbar1971optimal,Fabian:1973:AES,KushnerClark}.

With regards to finite-sample rates,
\citet{agarwal2012lowerbounds} provide information-theoretic lower
bounds (for any strategy) for certain stochastic convex optimization
problems. This result does not imply our bounds as they do not consider
the same smoothness assumptions on $\psi$.  For the special case of
linear least-squares regression, there are several upper bounds (for
instance, \citet{CD07,HKZ_regression}).  Recently,
\citet{shamir14squaredloss} provides lower bounds specifically for the
least-squares estimator, applicable under model mis-specification, and sharp only for specific problems.

\paragraph{Linearly convergent optimization (and approaches based on doubling).}
There are numerous algorithms for optimizing sums of convex functions
that converge linearly, \ie that depend only logarithmically on the
target precision.
Notably, several recently developed such algorithms are applicable
in the setting where the sample size $N$ becomes large, due to their
stochastic nature \citep{strohmer09kaczmarz, leroux2012sag,
shalevshwartz2013sdca, johnson2013svrg}.
These procedures minimize a sum of $N$ losses in time (near to) linear
in $N$, provided $N$ is sufficiently large relative to the dimension
and the condition number.

Naively, one could attempt to use one of these algorithms to directly
compute the ERM. Such an attempt poses two difficulties.
First, we would need to prove concentration results for the empirical
function $\widehat P_N(w) = \frac{1}{N} \sum_{i=1}^N \psi_i(w)$; in
order to argue that these algorithms perform well in linear time with
respect to the objective $P$, one must relate the condition number of
$\widehat P_N(w)$ to the condition number of $P(w)$.
Second, we would need new generalization analysis in order to
relate the in-sample error $\eps_N(\widehat w_N)$, where $\eps_N(w)
\defeq \widehat P_N(w) - \min_{w'} \widehat P_N(w')$, to the
generalization error $\E[P(\widehat w_N) - P(w_*)]$.
To use existing generalization analyses would demand that $\eps_N(w_N)
= \Omega(1/N)$, but the algorithms in question all require at least
$\log N$ passes of the data (furthermore scaled by other
problem-dependent factors) to achieve such an in-sample error.
Hence, this approach would not immediately describe the generalization
error obtained in time linear in $N$.
Finally, it requires that entire observed data sample, constituting
the sum, be stored in memory.

A second natural question is: can one naively use a doubling trick
with an extant algorithm to compete with the ERM?  By this we mean to
iteratively run such a linearly convergent optimization algorithm, on
increasingly larger subsets of the data, with the hope of cutting the
error at each iteration by a constant fraction, eventually down to
that of the ERM. There are two points to note for this
approach. First, the approach is not implementable in a streaming
model as one would eventually have to run the algorithm on a constant
fraction of the entire dataset size, thus essentially holding the
entire dataset in memory.  Second, proving such an algorithm succeeds
would similarly involve the aforementioned type of generalization
argument.

We conjecture that these tight generalization arguments described are
attainable, although with a somewhat involved analysis. For linear
regression, the bounds in \citet{HKZ_regression} may suffice. More
generally, we believe the detailed ERM analysis provided herein could
be used.

In contrast, the statistical convergence analysis of our single-pass
algorithm is self-contained and does not go through any generalization
arguments about the ERM. In fact, it avoids matrix concentration
arguments entirely.

\paragraph{Comparison to related work.} To our
knowledge, this work provides the first streaming algorithm guaranteed
to have a rate that approaches that of the ERM (under certain
regularity assumptions on $\D$), where the initial error is decreased
at a super-polynomial rate. The previous work, in the general case
that we consider, only provides asymptotic convergence
guarantees~\citep{polyak1992stoch}. For the special case of linear
least-squares regression, the concurrent and independent work presented in
\citet{dieuleveut2014non} and \citet{bach2014constant} also converges to the rate
of the ERM, with a lower-order dependence on the initial error of
$\Omega(1/N^2)$. Furthermore, even if we ignored memory constraints and
focused solely on computational complexity, our algorithm compares
favorably to using state-of-the-art algorithms for minimizing sums of
functions (such as the linearly convergent algorithms in
\citet{leroux2012sag, shalevshwartz2013sdca, johnson2013svrg}); as
discussed above, obtaining a convergence rate with these algorithms
would entail some further generalization analysis.

It would be interesting if one could quantify an approach of
restarting the algorithm of \citet{polyak1992stoch} to obtain guarantees comparable to our streaming algorithm.
Such an analysis could be delicate in settings other than
linear regression, as their learning rates do not decay too quickly or too
slowly (they must decay strictly faster than
$1/\sqrt{N}$, yet more slowly than $1/N$). In contrast, our
algorithm takes a constant learning rate to obtain its
constant competitive ratio. Furthermore,
our algorithm is easily parallelizable and its analysis, we believe, is
relatively transparent.

\subsection{Organization}
Section~\ref{sec:summary} summarizes our main results, and
Section~\ref{sec:applications} provides applications to a few standard
statistical models. Section~\ref{sec:analysis} provides the main
technical claims for our algorithm, Streaming
SVRG (Algorithm~\ref{alg:streaming_svrg}). Section~\ref{sec:erm} provides finite-sample rates for the ERM,
along with proofs for these rates. The Appendix contains various
technical lemmas and proofs of our corollaries.

\section{Main results}
\label{sec:summary}
This section summarizes our main results, as corollaries of more
general theorems provided later. After providing our assumptions in
Section~\ref{sec:summary:assumptions}, Section~\ref{sec:alg} provides
the algorithm, along with performance guarantees.  Then
Section~\ref{sec:summary:erm} provides upper and lower bounds of the
statistical rate of the empirical risk minimizer.

First, a few preliminaries and definitions are needed. Denote
$
\|x\|^2_M \defeq x^\T M x
$
for a vector $x$ and a matrix $M$ of appropriate dimensions. Denote
$\lambda_{\max}(M)$ and $\lambda_{\min}(M)$ as the maximal and minimal
eigenvalues of a matrix $M$. Let $\identityMatrix$ denote the identity matrix. Also, for positive semidefinite symmetric
matrices $A$ and $B$, $A \preceq B$ if and only if $x^\T A x \leq x^\T
B x$ for all $x$.

Throughout, define $\sigma^2$ as:
\begin{align}
\sigma^2 &\defeq
\fullSigmaSquared
\end{align}
This quantity governs the precise (problem dependent) convergence rate of the
ERM. Namely, under certain restrictions on $\D$, we have
\begin{align}
\lim_{N \rightarrow \infty}\frac{\E[ P(\widehat w^{\textrm{ERM}}_{N})
- P(w_*)]} {\sigma^2/N} = 1.
\end{align}
This limiting rate is well-established in asymptotic statistics (see,
for instance, \cite{vandervaart2000asymptotic}), whereas
Section~\ref{sec:summary:erm} provides upper and lower bounds on this
rate for finite sample sizes $N$. Analogous to the Cram\'er-Rao lower
bound, under certain restrictions, $\sigma^2/N$ is the asymptotically
efficient rate for stochastic approximation problems
~\citep{anbar1971optimal,Fabian:1973:AES,kushner2003stochastic}.\footnote{Though,
as with Cram\'er-Rao, this may be improvable with biased
estimators.}

The problem dependent rate of $\sigma^2/N$ sets the
benchmark. Statistically, we hope to achieve a leading order
dependency of $\sigma^2/N$ quickly,
with rapidly-decaying dependence on the initial error.

\subsection{Assumptions}
\label{sec:summary:assumptions}

We now provide two assumptions under which we analyze the convergence
rate of our streaming algorithm,
Algorithm~\ref{alg:streaming_svrg}. Our first assumption is relatively
standard. It provides upper and lower quadratic approximations (the
lower approximation is on the full objective $P$).

\begin{assumption}
\label{ass:main}
Suppose that:
\begin{enumerate}
\item The objective $P$ is twice differentiable.
\item (Strong convexity) The objective $P$ is $\mu$-strongly convex, \ie for all $w,w' \in \mathcal{S}$,
\begin{align}
P(w) \geq P(w')+ \nabla P(w')^\T (w-w') +\frac{\mu}{2} \|w-w'\|_2^2,
\end{align}
\item (Smoothness) Each loss $\psi$ is $L$-smooth (with probability one), \ie for all $w,w' \in \mathcal{S}$,
\begin{align}  \label{equation:smoothness}
\psi(w) \leq \psi(w')+ \nabla \psi(w')^\T (w-w') +\frac{L}{2} \|w-w'\|_2^2,
\end{align}
\end{enumerate}
\end{assumption}
Our results in fact hold under a slightly weaker version of this
assumption -- see Remark~\ref{remark:av_smoothness}. Define:
\begin{align}
\conditionNumber \defeq \frac{L}{\mu}.
\end{align}
The quantity $\conditionNumber$ can be interpreted as the condition
number of the optimization objective \eqnref{eq:P}. The following
definition quantifies a global bound on the Hessian.
\begin{definition}[$\alpha$-bounded Hessian]
\label{def:alpha}
Let $\alpha\geq 1$ be the smallest value (if it exists) such that
for all $w\in \mathcal{S}$, $\nabla^2 P(w_*) \preceq \alpha \nabla^2 P(w)$.
\end{definition}

Under Assumption~\ref{ass:main}, we have $\alpha \leq
\conditionNumber$, because $L$-smoothness implies $\hess P(w_*)
\preceq L \identityMatrix$ and $\mu$-strong convexity implies $\mu
\identityMatrix \preceq \hess P(w)$. However, $\alpha$ could be much
smaller. For instance, $\alpha=1$ in linear regression, whereas
$\kappa$ is the maximum to minimum eigenvalue ratio of the design
matrix.

Our second assumption offers a stronger, local relationship on the
objective's Hessian, namely \emph{self-concordance}.  A function is
self-concordant if its third-order derivative is bounded by a multiple
of its second-order derivative. Formally, $f:\R\to \R$ is
$\selfConcordance$ self-concordant if and only if $f$ is convex and
$\abs{f'''(x)} \le \selfConcordance f''(x)^{3/2}$.
A multivariate function $f:\R^d\to \R$ is
$\selfConcordance$ self-concordant if and only if its restriction to any line is
$\selfConcordance$ self-concordant.

\begin{assumption}[Self-concordance]
\label{ass:selfconcordance}
Suppose that:
\begin{enumerate}
\item $P$ is $\selfConcordance$-self concordant (or that the weaker
condition in Equation~\eqnref{eq:self-concordance} holds).\\
\item The following kurtosis condition holds:
\begin{align*}
\frac{\bigE {\psi\sim \D} { \| \nabla \psi(w_*) \|_2^4 }}
{\left(\bigE {\psi\sim \D} { \| \nabla \psi(w_*) \|_2^2 }\right)^2}
\leq \kurtosis
\end{align*}
\end{enumerate}
\end{assumption}

Note that these two assumptions are also standard assumptions in the
analysis of the two phases of Newton's method (aside from the kurtosis
condition): the first phase of Newton's method gets close to the
minimizer quickly (based on a global strong convexity assumption) and
the second phase obtains quadratic convergence (based on local
curvature assumptions on how fast the local Hessian changes,
\eg self-concordance). Moreover, our proof of the streaming algorithm
follows a similar structure; we use Assumption~\ref{ass:main} to
analyze the progress of our algorithm when the current point is far
away from optimality and Assumption~\ref{ass:selfconcordance} when it
is close.

\subsection{Algorithm}
\label{sec:alg}

\begin{algorithm}[t]
\begin{algorithmic}

\INPUT Initial point $\tilde w_0$, batch sizes $\{k_0,
k_1,\ldots\}$, update frequency $m$, learning rate $\eta$,
smoothness $L$

\FOR {each stage $s = 0, 1, 2, \dots$}
\STATE Sample $\tilde \psi_1, \dots, \tilde \psi_{k_s}$ from $\mD$
and compute the estimate
\begin{align}
\label{eq:alg:sample_gradient}
\widehat{\nabla P(\tilde w_{s})} &=
\frac{1}{k_s}  \sum_{i\in[k_s]} \nabla \tilde \psi_i(\tilde w_{s}).
\end{align}

\STATE Sample $\tilde m$ uniformly at random from $\set{1,2,\dots,m}$.

\STATE $w_0 \gets \tilde w_s$

\FOR {$t = 0, 1, \dots, \tilde m - 1$}
\STATE Sample $\psi_t$ from $\mD$ and set
\begin{align}
\label{eq:alg:step}
w_{t+1} &\gets
w_t-\frac{\eta}{L} \left( \nabla \psi_t(w_t) - \nabla \psi_t(\tilde
w_{s})+\widehat{\nabla P(\tilde w_{s})} \right).
\end{align}
\ENDFOR
\STATE $\tilde w_{s+1} \gets w_{\tilde m}$
\ENDFOR
\end{algorithmic}
\caption{Streaming Stochastic Variance Reduced Gradient (Streaming
SVRG)}
\label{alg:streaming_svrg}
\end{algorithm}

Here we describe a streaming algorithm and provide its convergence
guarantees.  Algorithm~\ref{alg:streaming_svrg} is inspired by the
Stochastic Variance Reduced Gradient (SVRG) algorithm of
\citet{johnson2013svrg} for minimizing a strongly convex sum of smooth
losses. The algorithm follows a simple framework that proceeds in
stages. In each stage $s$ we draw $k_s$ samples independently at
random from $\D$ and use these samples to obtain an estimate of the
gradient of $P$ at the current point, $\tilde{w}_s$
(\eqref{eq:alg:sample_gradient}). This \emph{stable} gradient, denoted
$\widehat{\grad P(\tilde w_{s})}$, is then used to decrease the
variance of a gradient descent procedure. For each of $\tilde{m}$
steps (where $\tilde{m}$ is chosen uniformly at random from $\{1, 2,
\ldots, m\}$), we draw a sample $\psi$ from $\D$ and take a step
opposite to its gradient at the current point, plus a zero-bias
correction given by $\grad \psi(\tilde w_s) - \widehat{\grad P(\tilde
w_s)}$ (see \eqref{eq:alg:step}).

The remainder of this section shows that, for suitable choices of
$k_s$ and $m$, Algorithm~\ref{alg:streaming_svrg} achieves desirable
convergence rates under the aforementioned assumptions.

\begin{remark}[Generalizing SVRG]
\normalfont
Note that Algorithm~\ref{alg:streaming_svrg}  is a generalization of
SVRG. In particular if we chose $k_s = \infty$, \ie if
$\widehat{\nabla P(\tilde w_{s})} = \nabla P(\tilde w_{s})$, then our
algorithm coincides with the SVRG algorithm of
\citet{johnson2013svrg}. Also, note that \citet{johnson2013svrg} do
not make use of any self-concordance assumptions.
\end{remark}

\begin{remark}[Non-conformance to stochastic first-order oracle models]
\normalfont
Algorithm~\ref{alg:streaming_svrg} is not implementable in
the standard stochastic first-order oracle model, \eg that
which is assumed in order to obtain the lower bounds in \citet{nemirovski1983problem} and \citet{agarwal2012lowerbounds}.
Streaming SVRG computes the gradient of the randomly drawn $\psi$ at
two points, while the oracle model only allows gradient
queries at one point.
\end{remark}

We have the following algorithmic guarantee under only
Assumption~\ref{ass:main}, which is a corollary of
Theorem~\ref{thm:alphaconverge} (also see the Appendix).

\begin{corollary}[Convergence under $\alpha$-bounded Hessians]
\label{corollary:main}
Suppose Assumption~\ref{ass:main} holds.
Fix $\tilde w_0\in\R^d$.
For $p\geq2$ and $\base\geq 3$, set
$\eta=\frac{1}{20 \base^{p+1}}$,
$m=\frac{20
\base^{p+1}\conditionNumber}{\eta}$,
$k_0 =20\alpha\kappa\base^{p+1}$, and
$k_s = b k_{s-1}$. Denote:
\begin{align*}
N_s &\defeq \sum_{\tau=0}^{s-1} (k_\tau+m)
\end{align*}
($N_s$ is an upper bound on the number of samples drawn up to the
end of stage $s$).

Let $\widehat w_{N_s}$ be the parameter returned at iteration $s$ by
Algorithm~\ref{alg:streaming_svrg}. For $N_s\geq
b^{p^2+6p} \conditionNumber$ (and so $s>p^2+6p$), we have
\begin{align*}
\EpError{\widehat w_{N_s}}
&\leq
\left(
\left(1+\frac{4}{b}\right) \frac{\sqrt{\alpha}\sigma}{\sqrt{N_s}} +
\sqrt{\frac{P(\tilde  w_0)-P(w_*)}{\left( \frac{N_s}{\alpha \conditionNumber} \right)^{p} }}
\right)^2
\end{align*}
\end{corollary}

When $\alpha=1$ (such as for least squares regression), the above
bound achieves the ERM rate of $\sigma^2/N$ (up to a constant factor,
which can be driven to one, as discussed later).  Furthermore, under
self-concordance, we can drive the competitive ratio
\eqnref{eq:competition-ratio} down from $\alpha$ to arbitrarily near
to $1$. The following is a corollary of
Theorem~\ref{thm:selfconcordanceconverge} (also see the Appendix):

\begin{corollary}[Convergence under self-concordance]
\label{corollary:selfconcordant}
Suppose Assumptions~\ref{ass:main} and
\ref{ass:selfconcordance} hold.
Consider $\tilde w_0\in\R^d$.
For $p\geq2$ and $\base\geq 3$, set
$\eta=\frac{1}{20 \base^{p+1}}$,
$m=\frac{20 \base^{p+1}\conditionNumber}{\eta}$,
$k_0 =\max\{400\conditionNumber^2\base^{2p+3}, 10 \kurtosis \} =\max\{b m
\conditionNumber, 10 \kurtosis \} $,
and $k_s = \base k_{s-1}$.
Denote $N_s \defeq \sum_{\tau=0}^{s-1} (k_s+m)$ (an upper bound on the
number of samples drawn up to the end of stage $s$).
Let $\widehat w_{N_s}$ be the parameter returned at iteration $s$ by
Algorithm~\ref{alg:streaming_svrg}.
Then:
\begin{eqnarray*}
\EpError{\widehat w_{N_s}}
&\leq
\left(
\left(1+\frac{5}{b}\right) \frac{\sigma}{\sqrt{N_s}} +
\left(2+\frac{5}{b}\right)\frac{\sqrt{\conditionNumber}\sigma}{\sqrt{N_s}} \min\bigset{1, \left(\frac{N_s}{2(M\sigma+1)^2k_0}\right)^{-p/2}} +
\sqrt{\frac{P(\tilde  w_0)-P(w_*)}{\left( \frac{N_s}{2k_0} \right)^{p+1} }}
\right)^2
\end{eqnarray*}
\end{corollary}

\begin{remark}[Implementation and parallelization]
\label{rem:parallel}
\normalfont
Note that Algorithm~\ref{alg:streaming_svrg} is simple to implement
and requires little space. In each iteration, the space usage is linear in the size of
a single sample (along with needing to count to $k_s$ and
$m$). Furthermore, the algorithm is easily parallelizable
once we have run enough stages. In both
Theorem~\ref{thm:alphaconverge} and
Theorem~\ref{thm:selfconcordanceconverge} as $s$ increases $k_s$ grows
geometrically, whereas $m$ remains constant. Hence,
the majority of the computation time is spent averaging the
gradient, \ie \eqref{eq:alg:sample_gradient}, which is easily
parallelizable.
\end{remark}

Note that the constants in the parameter settings for the Algorithm
have not been optimized.  Furthermore, we have not attempted to fully
optimize the time it takes the algorithm to enter the second phase (in
which self-concordance is relevant), and we conjecture that the
algorithm in fact enjoys even better dependencies. Our emphasis is on
an analysis that is flexible in that it allows for a variety of
assumptions in driving the competitive ratio to $1$ (as is done in the
case of logistic regression in Section~\ref{sec:applications}, where
we use a slight variant of self-concordance).

Before providing statistical rates for the ERM, let us remark that the
above achieves super-polynomial convergence rates and that the
competitive ratio can be driven to $1$ (recall that $\sigma^2/N$ is
the rate of the ERM).

\begin{remark}[Linear convergence and super-polynomial convergence]
\normalfont
Suppose the ratio $\gamma$ between $P(\tilde{w}_0) - P(w_*)$ and
$\sigma^2$ is known approximately (within a multiplicative factor),
we can let $k_s = k_0$ for $\log_\base \gamma$ number of iterations,
then start increasing $k_s = \base k_{s-1}$. This way in the first
$\log_\base \gamma$ iterations $\E[P(\hat{w}_{N_s})-P(w_*)]$ is
decreasing geometrically.
Furthermore, even without knowing the ratio $\gamma$, we can can
obtain a super-polynomial rate of convergence by setting the
parameters as we specify in the next remark. (The dependence on the
initial error will then be $2^{-\Omega(\log N/\log\log N)^2}$.)
\end{remark}

\begin{remark}[Driving the ratio to $1$]
\normalfont
By choosing $b$ sufficiently large, the competitive ratio
\eqnref{eq:competition-ratio} can be made close to $1$ (on every
problem). Furthermore, we can ensure this constant goes to $1$ by
altering the parameter choices adaptively: let $k_s = 4^s(s!)  k_0$,
and let $\eta_s = \eta/2^s$, $m_s = m \cdot 4^s$. Intuitively, $k$
grows so fast that $\lim_{s\to \infty} k_s/N_s = 1$; $\eta_s$ and
$m_s$ are also changing fast enough so the initial error vanishes very
quickly.
\end{remark}

\subsection{Competing with the ERM}
\label{sec:summary:erm}

Now we provide a finite-sample characterization of the rate of
convergence of the ERM under regularity conditions.
This essentially gives the numerator of \eqref{eq:competition-ratio},
allowing us to compare the rate of the ERM against the rate achieved
by Streaming SVRG.
We provide the more general result in Theorem~\ref{theorem:ERM}; this
section focuses on a corollary.

In the following, we constrain the domain $\mathcal{S}$; so the ERM,
as defined in~\eqref{eq:ERM_definition}, is taken over this restricted
set. Further discussion appears in Theorem~\ref{theorem:ERM} and the
comments thereafter.

\begin{corollary}[of Theorem~\ref{theorem:ERM}]
\label{corollary:ERM}
Suppose $\psi_1,\psi_2,\ldots,\psi_N$ are an independently drawn
sample from $\mathcal{D}$. Assume the following regularity conditions
hold; see Theorem~\ref{theorem:ERM} for weaker conditions.
\begin{enumerate}
\item $\mathcal{S}$ is compact.
\item $\psi$ is convex  (with probability one).
\item $w_*$ is an interior point of $\mathcal{S}$, and $\nabla^2 P(w_*)$ exists and is positive definite.
\item (Smoothness) Assume the first, second, and third
derivatives of $\psi$ exist and are uniformly bounded on $\mathcal{S}$.
\end{enumerate}

Then, for the ERM $\widehat w^{\textrm{ERM}}_{N}$
(as defined in~\eqref{eq:ERM_definition}), we have
\begin{align*}
\lim_{N \rightarrow \infty}\frac{\E[ P(\widehat w^{\textrm{ERM}}_{N})
- P(w_*)]} {\sigma^2/N} = 1
\end{align*}

In particular, the following lower and upper bounds hold. With problem
dependent constants $C_0$ and $C_1$ (polynomial in the
relevant quantities, as specified in Theorem~\ref{theorem:ERM}), we have
for all $p\geq 2$, if $N$ satisfies $\frac{p\log
d N}{N}\leq C_0 $, then
\begin{align*}
\left(1- C_1 \sqrt{\frac{p\log
d N}{N}}\right) \frac{\sigma^2}{N}
&\leq
\E[ P(\widehat w^{\textrm{ERM}}_{N}) - P(w_*)] \\
&\leq
\left(1+C_1 \sqrt{\frac{p\log
d N}{N}}\right) \frac{\sigma^2}{N} + \frac{\max_{w\in
\mathcal{S}} \left(P(w)-P(w_*)\right)}{N^p}
\end{align*}
\end{corollary}

\section{Applications: one pass learning and generalization}
\label{sec:applications}

This section provides applications to a few standard statistical
models, in part providing a benchmark for comparison on concrete
problems.
For the widely studied problem of least-squares regression, we
also instantiate upper and lower bounds for the ERM.
The applications in this section can be extended to include
generalized linear models, some $M$-estimation problems, and other loss functions (\eg the Huber
loss).

\subsection{Linear least-squares regression} \label{subsec:regression}
In linear regression, the goal is to minimize the
(possibly $\ell_2$-regularized) squared loss $\psi_{X,Y}(w) =\allowbreak
(Y-w^\T X)^2 + \lambda \norm{w}_2^2$ for a random data point $(X, Y)
\in \R^d \by \R$.
The objective \eqnref{eq:P} is
\begin{align}
P(w) &= \E_{X,Y \samp \mD}\bigsqbra{ (Y-w^\top X)^2 } + \lambda
\norm{w}_2^2 \, .
\end{align}

\subsubsection{Upper bound for the algorithm}
Using that $\alpha=1$, the following corollary illustrates that
Algorithm~\ref{alg:streaming_svrg} achieves the rate of the ERM,
\begin{corollary}[Least-squares performance of streaming SVRG]
\label{cor:lls}
Suppose that
$\|X\|^2\leq L$. Define $\mu =
\lambda+\lambda_{\min}(\Sigma)$.  Using the parameter settings of
Theorem~\ref{corollary:main} and supposing that $N\geq
b^{p^2+6p}\conditionNumber$,
\begin{align*}
\EpError{\tilde w_{N}}
&\leq
\left(\left(1+\frac{4}{b}\right) \frac{\sigma}{\sqrt{N}} +
\sqrt{\frac{P(\tilde  w_0)-P(w_*)}{\left( \frac{N}{\conditionNumber}  \right)^p} }\right)^2
\end{align*}
\end{corollary}

\begin{remark}[When $N\leq \conditionNumber$]
\normalfont
If the sample size is less than $\conditionNumber$ and
$\lambda=0$, there exist distributions on $X$ in which the ERM is
not unique (as the sample matrix $\frac{1}{N}\sum X_i
X_i^\top$ will not be invertible, with reasonable probability, on
these distributions by construction).
\end{remark}

\begin{remark}[When do the streaming SVRG bounds become meaningful?]
\normalfont
Algorithm~\ref{alg:streaming_svrg} is competitive with the
performance of the ERM when the sample size $N$ is slightly larger
than a constant times $\conditionNumber$. In
particular, as the sample $N$ size grows larger than $\conditionNumber$,
then the initial error is decreased
at an arbitrary polynomial rate in $N/\conditionNumber$.
\end{remark}

Let us consider a few special cases. First, consider the unregularized
setting where $\lambda=0$. Assume also that the least-squares problem
is well-specified. That is, $Y=w_*^\top X +\eta$ where $\E[\eta]=0$
and $\E[\eta^2]=\sigma^2_{\textrm{noise}}$. Define $\Sigma
= \E[XX^\top]$.
Here, we have
\begin{align}
\sigma^2 = \E \|\eta^2 X\|^2_{\Sigma^{-1}} = d \sigma^2_{\textrm{noise}}.
\end{align}
In other words, Corollary~\ref{cor:lls} recovers the classical rate in
this case.

In the mis-specified case~-- where we do not assume the aforementioned
model is correct (\ie $\E[Y|X]$ may not equal $w_*^\top X$ )~-- define
$Y_*(X) =w_*^\top X$, and we have
\begin{align}
\sigma^2
&=
\bigE {} { (Y-Y_*(X))^2\| X\|^2_{\Sigma^{-1}} } \label{sigma:least_squares} \\
&=
\bigE {} { (Y-\E[Y \mid X])^2\| X\|^2_{\Sigma^{-1}} } + \bigE {} { (\E[Y \mid X]-Y_*(X))^2\| X\|^2_{\Sigma^{-1}} } \\
&=
\bigE {} { \var(Y \mid X)\| X\|^2_{\Sigma^{-1}} } + \bigE {} { \bias(X)^2
\| X\|^2_{\Sigma^{-1}} }
\end{align}
where the last equality exposes the effects of the approximation error:
\begin{align}
\var(Y \mid X) &\defeq \E[(Y-\E[Y \mid X])^2 \mid X]
&\text{ and }&&
\bias(X) &\defeq \E[Y \mid X]-Y_*(X).
\label{eq:lls-bias-var}
\end{align}

In the regularized setting (\aka ridge regression)~-- also not
necessarily well-specified~-- we have
\begin{align}
\sigma^2 &= \E[ \|(Y-Y_*(X)) X+\lambda w_*\|^2_{(\Sigma+\lambda \mathrm{I})^{-1}} ]
\label{eq:sigma-ridge}
\end{align}

\subsubsection{Statistical upper and lower bounds}
For comparison, the following corollary (of Theorem~\ref{theorem:ERM})
provides lower and upper bounds for the statistical rate of the ERM.
\begin{corollary}[Least-squares ERM bounds]
Suppose that $\|X\|^2_{(\Sigma+\lambda \mathrm{I})^{-1}}\leq
\widetilde \conditionNumber$ and the dimension is $d$ (in the infinite dimensional setting, we
may take $d$ to be the intrinsic dimension, as per
Remark~\ref{remark:dimension}).  Let $c$ be an appropriately chosen
universal constant.  For all $p>0$, if
$\frac{p \log N}{N}\leq \frac{c}{\widetilde\conditionNumber}$, then
\begin{align*}
\EpError{\widehat w^{\textrm{ERM}}_{N}}
&\geq
\left(1-c \sqrt{\frac{\widetilde\conditionNumber p\log  d N}{N}}\right) \frac{\sigma^2}{N}
- \frac{\sqrt{\E\left[Z^4\right]}}{N^{p/2}}
\end{align*}
where $Z=\left\|\nabla \psi(w_*)\right\|_{(\nabla^2
P(w_*))^{-1}}=\|(Y-w_*^\top X) X+\lambda w_*\|_{(\Sigma+\lambda
I)^{-1}}$.

For an upper bound, we have
two cases:
\begin{itemize}
\item (Unregularized case) Suppose $\lambda=0$. Assume that we constrain the ERM to lie
in some compact set $\mathcal{S}$ (and supposing $w_*\in\mathcal{S}$).  Then for all $p>0$, if
$\frac{p \log N}{N}\leq \frac{c}{\widetilde \conditionNumber}$, we have
\begin{align*}
\E[ P(\widehat w^{\textrm{ERM}}_{N}) - P(w_*)]
\leq
\left(1+c \sqrt{\frac{\widetilde \conditionNumber p\log  d N}{N}}\right) \frac{\sigma^2}{N}
+ \frac{\max_{w\in \mathcal{S}} \left(P(w)-P(w_*)\right)}{N^p}
\end{align*}
\item (Regularized case) Suppose $\lambda>0$. Then for all $p>0$, if
$\frac{p \log N}{N}\leq \frac{c}{\widetilde \conditionNumber}$, we have
\begin{align*}
\E[ P(\widehat w^{\textrm{ERM}}_{N}) - P(w_*)]
\leq
\left(1+c \sqrt{\frac{\widetilde \conditionNumber p\log  d N}{N}}\right) \frac{\sigma^2}{N}
+ \frac{\frac{\lambda_{\max}(\Sigma+\lambda)}{\lambda}\sigma^2}{N^p}
\end{align*}
(this last equation follows from a modification of the argument in
Equation~\eqref{eq:erm_upper}).
\end{itemize}
\end{corollary}

\begin{remark}[ERM comparisons]
\normalfont
Interestingly, for the upper bound (when $\lambda=0$), we see no way
to avoid constraining the ERM to lie in some compact set; this
allows us to bound the loss $P$
in the event of some extremely low probability failure (see
Theorem~\ref{theorem:ERM}).
The ERM upper bound has a term comparable to the initial
error of our algorithm. In contrast, the lower bound is for the
usual unconstrained least-squares estimator.
\end{remark}

\subsection{Logistic regression}

In (binary) logistic regression, we have a distribution on $(X,Y) \in
\R^d \by \{0,1\}$. For any $w$, define
\begin{align}
\Pr(Y=y \mid w,X) &\defeq \frac {\exp(yX^\T w)} {1+\exp(X^\T w)}
\end{align}
for $X \in \R^d$ and  $y \in \set{0,1}$. We do not assume the best
fit model $w_*$ is correct. The loss function is taken to be the regularized
log likelihood
$\psi_{X,y}(w) = -\log \Pr(Y \mid w,X) + \lambda \norm{w}_2^2$ and the objective \eqnref{eq:P}
instantiates as the negative expected (regularized) log likelihood $P(w) = \E[ -\log
\Pr(Y \mid w,X)]+ \lambda \norm{w}_2^2$.
Define $Y_*(X) = \Pr(Y=1 \mid w_*,X)$ and $\Sigma_* = \nabla^2P(w_*) =
\E[ Y_*(X) (1-Y_*(X)) X X^\top]+\lambda \mathrm{I}$.  Analogous to the least-squares case, we
can interpret $Y_*(X)$ as the conditional expectation of $Y$ under the
(possibly mis-specified) best fit model. With this notation, $\sigma^2$ is
similar to its instantiation under regularized least-squares (Equation~\eqnref{eq:sigma-ridge}):
\begin{align}
\sigma^2
&= \bigE {} { \frac{1}{2} \|(Y-Y_*(X)) X+\lambda w_*\|^2_{\Sigma_*^{-1}} }
\end{align}

Under this definition of $\sigma^2$, by
Theorem~\ref{corollary:selfconcordant} together with the following
defined quantities, the single-pass estimator of
Algorithm~\ref{alg:streaming_svrg} achieves a rate competitive with
that of the ERM:
\begin{corollary}[Logistic regression performance]
Suppose that $\|X\|^2\leq L$. Define
$\mu = \lambda$ and $\selfConcordance=\alpha\E[\|X\|^3_{(\nabla^2
P(w_*))^{-1}}]$.
Under parameters from Theorem~\ref{corollary:main}, we have
\begin{eqnarray*}
\EpError{\widehat w_{N}}
&\leq
\left(
\left(1+\frac{5}{b}\right) \frac{\sigma}{\sqrt{N}} +
\left(2+\frac{5}{b}\right)\frac{\sqrt{\conditionNumber}\sigma}{\sqrt{N}} \min\bigset{1, \left(\frac{N}{2(M\sigma+1)^2k_0}\right)^{-p/2}} +
\sqrt{\frac{P(\tilde  w_0)-P(w_*)}{\left( \frac{N}{2k_0} \right)^{p+1} }}
\right)^2
\end{eqnarray*}
\end{corollary}
The corollary uses Lemma~\ref{lemma:selfcorcordantbound_logistic}, a
straightforward lemma to
handle self-concordance for logistic regression, which is included for
completeness. See \citet{bach2010self} for techniques for analyzing the self-concordance
of logistic regression.

\section{Analysis of Streaming SVRG}
\label{sec:analysis}

Here we analyze
Algorithm~\ref{alg:streaming_svrg}. Section~\ref{sec:analysis:helper}
provides useful common lemmas. Section~\ref{sec:analysis:progress}
uses these lemmas to characterize the behavior of the
Algorithm~\ref{alg:streaming_svrg}. These are then used to prove convergence in
terms of both $\alpha$-bounded Hessians
(Section~\ref{sec:analysis:first_assumption}) and
$\selfConcordance$-self-concordance
(Section~\ref{sec:analysis:second_assumption}).

\subsection{Common lemmas}
\label{sec:analysis:helper}

Our first lemma is a consequence of smoothness. It is
the same observation made in \citet{johnson2013svrg}.

\begin{lemma} \label{lemma:smoothness}
If $\psi$ is smooth (with probability one), then
\begin{align}
\label{lemma:smoothness:equation:av_smoothness:}
\bigE {\psi \sim \D} { \| \nabla \psi(w) - \nabla \psi(w_*) \|^2 }
&\leq 2 L \bigpar{ \pError{w} }.
\end{align}
\end{lemma}

\begin{remark}[A weaker smoothness assumption]
\label{remark:av_smoothness}
Instead of the smoothness Assumption~\ref{ass:main} in
Equation~\ref{equation:smoothness}, it suffices to directly assume
\eqref{equation:av_smoothness}
and still have all results hold as presented. In doing so, we incur an
additional factor of $2$ as in this case we have $\hess P(w_*) \preceq
2 L \identityMatrix $ by Lemma~\ref{lem:hessian_opt_bound}. For
further explanation see Appendix~\ref{sec:weaker}.
\end{remark}

\begin{proof}
For an $L$-smooth function $f: \R^d \rightarrow \R$, we have
\begin{align}
f(w)-\min_{w'} f(w')
&\geq \frac{1}{2L}\|\nabla f(w)\|^2.
\end{align}
To see this, observe that
\begin{align*}
\min_{w'} f(w') &\leq \min_\eta f(w-\eta \nabla f(w)) \\
&\leq \min_\eta \left(
f(w)-\eta \|\nabla f(w)\|^2+\frac{1}{2}\eta^2L\|\nabla f(w)\|^2\right)
=f(w) - \frac{1}{2L}\|\nabla f(w)\|^2
\end{align*}
using the definition of $L$-smoothness.

Now define:
\begin{align}
g(w) &= \psi(w)-\psi(w_*)-(w-w_*)^\top \nabla \psi(w_*).
\end{align}
Since $\psi$ is $L$-smooth (with probability one) $g$ is $L$-smooth (with probability one) and it follows that:
\begin{align*}
\| \nabla \psi(w) - \nabla \psi(w_*) \|^2 & = \|\nabla g(w)\|^2 \\
& \leq 2L (g(w) - \min_{w'} g(w'))\\
& \leq 2L (g(w) - g(w_*)) \\
& = 2L( \psi(w)-\psi(w_*)-(w-w_*)^\top \nabla \psi(w_*) )
\end{align*}
where the second step follows from smoothness.  The proof is completed
by taking expectations and noting that $\E[\nabla \psi(w_*)]=\nabla P(w_*)=0$.
\end{proof}

Our second lemma bounds the variance of $\psi \sim \D$ in the $\hessOptInv$ norm.

\begin{lemma} \label{lemma:variance} Suppose Assumption~\ref{ass:main}
holds. Let $w\in \R^d$ and let $\psi \sim \D$. Then
\begin{align}
\E \left\| \nabla  \psi(w) - \nabla P(w) \right\|^2_{(\nabla^2 P(w_*))^{-1}}
\leq
2 \left( \sqrt{\conditionNumber \left(\pError{w}\right)}
+ \sigma \right)^2 \, .
\end{align}
\end{lemma}

\begin{proof}
For random vectors $a$ and $b$, we have
\[
\E \|a+b\|^2 = \E \|a\|^2 + 2\E a\cdot b + \E\|b\|^2
\leq  \E\|a\|^2 + 2\sqrt{\E \|a\|^2 \E \|b\|^2} + \E\|b\|^2
= \left(\sqrt{\E\|a\|^2}+ \sqrt{\E\|b\|^2}\right)^2
\]
Consequently,
\begin{align*}
&
\E \left\| \nabla  \psi(w) - \nabla P(w) \right\|^2_{(\nabla^2 P(w_*))^{-1}}
\\
& \leq
\left( \sqrt{\E \left\| \nabla  \psi(w) -\nabla  \psi(w_*) - \nabla P( w) \right\|^2_{(\nabla^2 P(w_*))^{-1}}}
+\sqrt{ \E \left\| \nabla  \psi(w_*) \right\|^2_{(\nabla^2  P(w_*))^{-1}}} \right)^2 \\
& \leq
\left( \sqrt{\frac{1}{\mu}\E \left\| \nabla  \psi( w) -\nabla  \psi(w_*) - \nabla P( w) \right\|^2}
+\sqrt{2}\sigma \right)^2
\end{align*}
where the last step uses $\mu\mathrm{I} \preceq \nabla^2 P(w_*)$ and
the definition of $\sigma^2$.

Observe that
\[
\E \left[\nabla  \psi( w) - \nabla  \psi(w_*) \right]
= \nabla P(w) - \nabla P(w_{*})
= \nabla P(w) \, .
\]
Applying Lemma~\ref{lemma:smoothness} and for random $a$, that $\E\|a - \E a\|^2 \leq \E\|a\|^2$,
we have
\[
\E \left\| \nabla  \psi(w) -\nabla  \psi(w_*) - \nabla P( w) \right\|^2
\leq
\E \left\| \nabla  \psi(w) -\nabla  \psi(w_*) \right\|^2
\leq 2 L(P(w) - P(w_*)).
\]
Combining and using the definition of $\conditionNumber$ yields the result.
\end{proof}

\subsection{Progress of the algorithm}
\label{sec:analysis:progress}

The following bounds the progress of one step of Algorithm~\ref{alg:streaming_svrg}.

\begin{lemma}\label{lem:contraction_one_step}
Suppose Assumption~\ref{ass:main} holds, $\tilde w_0 \in \R^d$, and
$\tilde \psi_1, \ldots \tilde \psi_k$ are functions from
$\R^d \rightarrow \R$. Suppose $\psi_1, \ldots \psi_m$ are sampled
independently from $\D$. Set $w_0=\tilde w_0$ and for $t\in
\{0,1,\ldots m-1\}$, set:
\[
w_{t+1} \defeq w_t -
\frac{\stepParameter}{L}
\left(\nabla \psi_t(w_t) - \nabla \psi_t(\tilde{w}_0) +
\frac{1}{k} \sum_{i \in [k]} \nabla \tilde{\psi}_i(\tilde{w}_0)
\right)
\]
for some $\eta>0$.
Define:
\[
\Delta \defeq \frac{1}{k} \sum_{i \in [k]} \nabla \tilde{\psi}_i
(\tilde{w}_0) - \nabla P(\tilde{w}_0) \, .
\]
For all $t$ let $\alpha_t$ be such that
\begin{equation}
\label{eq:local_strong_convexity}
P(w_*) \geq P(w_t) + (w_* - w_t)^\top \nabla P(w_t) + \frac{1}{2\alpha_t} \| w_t - w_*\|_{\nabla^2 P(w_*)}^2
\end{equation}
(note that such an $\alpha_t$ exists by Assumption~\ref{ass:main}, as
$\alpha_t \leq \kappa$).

Then for all $t$ we have
\begin{align}
\E L \|w_{t+1}-w_*\|^2
&\leq
\E \left[ L
\|w_t - w_*\|^2 -2 \stepParameter (1 - 4 \stepParameter) \left(P(w_t) -P(w_*)\right)
+ 8 \stepParameter^2 \left(P(\tilde{w}_0) -P(w_*)\right)
\right.
\nonumber
\\
&\enspace
\left.
+ \left(\alpha_t \stepParameter + 2 \stepParameter^2\right)
\left\|\Delta\right\|^2 _{(\nabla^2  P(w_*))^{-1}}
\right]
\label{eq:progres_lemma_recurrence}
\end{align}
\end{lemma}

\begin{proof}
Letting
\[
g_t(w) = \psi_t(w)- w^\top  \left( \nabla \psi_t(\tilde w_0) - \frac{1}{k}
\sum_{i\in[k]} \nabla \tilde \psi_i(\tilde w_0) \right)
\]
and recalling the definition of $w_{t + 1}$ and $\Delta$ we have
\begin{eqnarray}
\hspace*{-0.2in}\E_{\psi_t\sim\D} \|w_{t+1}-w_*\|^2
& = &
\E_{\psi_t\sim\D}\|w_t - w_*- \frac{\stepParameter}{L} \nabla g_t(w_t)\|^2
\nonumber
\\
& = & \E_{\psi_t\sim\D}\left[ \|w_t - w_*\|^2 -2 \frac{\stepParameter}{L} (w_t - w_*)^\top \nabla g_t(w_t)
+ \frac{\stepParameter^2}{L^2} \|\nabla g_t(w_t)\|^2 \right]
\nonumber
\\
& = & \|w_t - w_*\|^2 -2 \frac{\stepParameter}{L} (w_t - w_*)^\top
(\nabla P(w_t) + \Delta)
+ \frac{\stepParameter^2}{L^2} \E_{\psi_t\sim\D}\|\nabla g_t(w_t)\|^2
\label{eq:alg_lemma_1}
\end{eqnarray}
Now by \eqref{eq:local_strong_convexity} we know that
\begin{equation}
\label{eq:alg_lemma_2}
-2 (w_t-w_*)^\top \nabla P(w_t) \leq
-2(P(w_t) -P(w_*)) - \frac{1}{\alpha_t}\|w_t-w_*\|^2 _{\nabla^2  P(w_*)}
\, .
\end{equation}
Using Cauchy-Schwarz and that $2 a \cdot b \leq a^2 + b^2$ for scalar $a$ and $b$, we have
\begin{equation}
\label{eq:alg_lemma_3}
-2 (w_t - w_*)^\top \Delta
\leq \frac{1}{\alpha_t} \|w_t - w_*\|^2 _{\nabla^2  P(w_*)}
+ \alpha_t \left\| \Delta \right\|_{(\nabla^2  P(w_*))^{-1}}^2 \, .
\end{equation}
Furthermore
\begin{align*}
\E_{\psi_t\sim\D}\|\nabla g_t(w_t)\|^2
& = \E_{\psi_t\sim\D} \left\|\nabla \psi_t(w_t) - \nabla \psi_t(\tilde w_0) + \frac{1}{k} \sum_{i\in[k]} \nabla \tilde \psi_i(\tilde w_0)  \right\|^2
\\
& = \E_{\psi_t\sim\D} \left\|
\left(\nabla \psi_t(w_t) - \nabla \psi_t(w_*)\right)
- \left(\nabla \psi_t(\tilde w_0) - \nabla \psi_t(w_*)- \nabla
P(\tilde w_0)\right)
+ \Delta
\right\|^2\\
& \leq 2\E_{\psi_t\sim\D} \left\|
\left(\nabla \psi_t(w_t) - \nabla \psi_t(w_*)\right)
- \left(\nabla \psi_t(\tilde w_0) - \nabla \psi_t(w_*)- \nabla
P(\tilde w_0)\right)\|^2
+ 2\|\Delta \right\|^2\\
& \leq 4\E_{\psi_t\sim\D} \left\|\nabla \psi_t(w_t) - \nabla \psi_t(w_*)\right\|^2
+ 4\E_{\psi_t\sim\D} \left\|\nabla \psi_t(\tilde w_0) - \nabla
\psi_t(w_*)- \nabla P(\tilde w_0)\right\|^2
+ 2\|\Delta \|^2\\
& \leq
4\E_{\psi_t\sim\D} \left\|\nabla \psi_t(w_t) - \nabla \psi_t(w_*)\right\|^2+
4\E_{\psi_t\sim\D} \left\| \nabla \psi_t(\tilde w_0) - \nabla \psi_t(w_*) \right\|^2
+2\left\| \Delta\right\|^2
\end{align*}
where we have used that $\E[\nabla \psi_t(\tilde w_0) - \nabla
\psi_t(w_*)- \nabla P(\tilde w_0)]=0$ and $\E\|a - \E a\|^2 \leq \E\|a\|^2$.
Applying Lemma~\ref{lemma:smoothness} and using $\hess P(w_*) \preceq  L \identityMatrix$ yields
\begin{equation}
\label{eq:alg_lemma_4}
\E_{\psi_t\sim\D}\|\nabla g_t(w_t)\|^2
\leq
8 L \left(P(w_t)-P(w_*)\right) + 8 L \left(P(\tilde w_0)-P(w_*)\right)
+ 2L \left\| \Delta \right\|^2_{(\nabla^2 P(w_*))^{-1}} \, .
\end{equation}

Combining \eqref{eq:alg_lemma_1}, \eqref{eq:alg_lemma_2},  \eqref{eq:alg_lemma_3}, and \eqref{eq:alg_lemma_4} yields
\begin{align*}
\E_{\psi_t\sim\D} \|w_{t+1}-w_*\|^2
&\leq
\|w_t - w_*\|^2 -2 \frac{\stepParameter}{L} (1 - 4 \stepParameter) \left(P(w_t) -P(w_*)\right)
+ 8 \frac{\stepParameter^2}{L} \left(P(\tilde{w}_0) -P(w_*)\right)
\\
&\enspace + \left(\alpha_t \frac{\stepParameter}{L} + 2 \frac{\stepParameter^2}{L} \right)
\left\|\Delta\right\|^2 _{(\nabla^2  P(w_*))^{-1}},
\end{align*}
and multiplying both sides by $L$ yields the result.
\end{proof}

Finally we bound the progress of one stage of Algorithm~\ref{alg:streaming_svrg}.

\begin{lemma}\label{lem:contraction_multiple_steps}
Under the same assumptions as Lemma~\ref{lem:contraction_one_step},
for $\tilde m$ chosen uniformly at random in $\{1,\ldots m\}$ and $\tilde
w_1 \defeq w_{\tilde m}$, we have
\[
\E[\pError{\tilde w_1}]
\leq
\frac{1}{1 - 4 \stepParameter}
\left[
\left(
\frac{\kappa}{m \stepParameter} + 4 \stepParameter \right)
\pError{\tilde w_0}
+
\E\left[
\frac{\alpha_{\tilde m} + 2 \stepParameter}{2}
\right]
\| \Delta \|_{\hessOptInv}^{2}
\right]
\]
where we are conditioning on $\tilde w_0 $ and $\tilde \psi_1, \ldots
\tilde \psi_k$.
\end{lemma}

\begin{proof}
Taking an unconditional expectation with respect to
$\{\psi_t\}$ and summing \eqref{eq:progres_lemma_recurrence} from
Lemma~\ref{lem:contraction_one_step} from $t=m-1$ down to $t=0$
yields
\begin{eqnarray*}
L \cdot
\E \|w_m-w_*\|^2
& \leq & L \cdot \|\tilde w_0 - w_*\|^2
-2 \stepParameter (1-4 \stepParameter) \sum_{t=0}^{m-1} \E\left(P(w_t)  -P(w_*)\right) \\
&&
8m\stepParameter^2 \left( \E P(\tilde w_0)-P(w_*) \right)
+
\sum_{t = 0}^{m - 1} \E\left[\left(\alpha_t \eta + 2 \stepParameter^2 \right)
\left\|\Delta\right\|^2_{(\nabla^2  P(w_*))^{-1}} \right]
\end{eqnarray*}
By strong convexity,
\[
\|\tilde w_0 - w_*\|^2 \leq \frac{2}{\mu} \left( P(\tilde
w_0)-P(w_*) \right)
\]
and a little manipulation yields that:

\begin{eqnarray*}
\frac{2 \stepParameter (1 - 4 \stepParameter)}{m} \sum_{t=0}^{m-1} \E\left(P(w_t)  -P(w_*)\right)
&
\leq & \left(\frac{2 \conditionNumber}{m} + 8 \stepParameter^2 \right) \left( P(\tilde w_0)-P(w_*) \right) \\
&&
+  \sum_{t = 0}^{m - 1} \E\left[ \frac{\alpha_t \stepParameter + 2 \stepParameter^2}{m}
\E \left\|\Delta\right\|^2_{(\nabla^2  P(w_*))^{-1}} \right]
\end{eqnarray*}
Rearranging terms and applying the definition of $\tilde w_1$ then yields the result.
\end{proof}

\subsection{With $\alpha$-bounded Hessians}
\label{sec:analysis:first_assumption}

Here we prove the progress made by Algorithm~\ref{alg:streaming_svrg}
in a single stage under only Assumption~\ref{ass:main}.

\begin{theorem}[Stage progress with $\alpha$-bounded Hessians]
\label{thm:alphaconverge}
Under Assumption~\ref{ass:main}, for
Algorithm~\ref{alg:streaming_svrg}, we have for all $s$:
\begin{align*}
& \E[\pError{\tilde w_{s + 1}}] \nonumber\\
\leq &
\frac{1}{1 - 4\eta }
\left[\left(\frac{\conditionNumber}{m \eta} + 4\eta \right) \E[\pError{\tilde w_s}]
+ \frac{\alpha + 2 \eta }{k}
\left(\sqrt{\conditionNumber \cdot \E[\pError{\tilde{w}_s}]} + \sigma
\right)^2
\right] \, .
\end{align*}
\end{theorem}

\begin{proof}
By definition of $\alpha$, we have $\alpha_t \leq \alpha$ for all
$t$ in Lemma~\ref{lem:contraction_multiple_steps} and therefore
\[
\E[\pError{\tilde w_{s + 1}}]
\leq
\frac{1}{1 - 4 \stepParameter}
\left[
\left(
\frac{\kappa}{m \stepParameter} + 4 \stepParameter \right)
\E[\pError{\tilde w_s}]
+
\frac{\alpha + 2 \stepParameter}{2}
\E\left[
\| \Delta \|_{\hessOptInv}^{2}
\right]
\right]
\]
Now using that the $\tilde{\psi}_i$ are independent and that $\E[\nabla \tilde{\psi}_i (\tilde{w}_s)] = \nabla P(\tilde{w}_s)$ we have
\begin{align*}
\E[\|\Delta\|_{\hessOptInv}^2]
&=
\frac{1}{k} \bigE {\psi \sim \D} { \|\nabla \tilde{\psi}_1 (\tilde{w}_s) - P(\tilde{w}_s)\|_{\hessOptInv}^2 } \\
&\leq
\frac{2}{k} \E \left[ \conditionNumber (\pError{\tilde{w}_s})
+\sigma \sqrt{\conditionNumber (\pError{\tilde{w}_s}}) + \sigma^2
\right] \\
&\leq
\frac{2}{k} \left[ \conditionNumber \E [\pError{\tilde{w}_s}]
+\sigma \sqrt{\conditionNumber \E [\pError{\tilde{w}_s}]} + \sigma^2 \right] \\
&=
\frac{2}{k} \bigpar{ \sqrt{\conditionNumber \cdot \E[\pError{\tilde{w}_s}]} + \sigma }^2
\end{align*}
where we have also used
Lemma~\ref{lemma:variance} and Jensen's inequality.
\end{proof}

\subsection{With $\selfConcordance$-self-concordance}
\label{sec:analysis:second_assumption}

Our main result in the self-concordant case follows.

\begin{theorem}[Convergence under self-concordance]
\label{thm:selfconcordanceconverge}
Suppose Assumption~\ref{ass:main} and \ref{ass:selfconcordance}
hold. Under Algorithm~\ref{alg:streaming_svrg}, for $\stepParameter
\leq \frac{1}{8}$, $k \geq 10\kurtosis $, and all $s$, we have
\begin{align*}
\E[\pError{\tilde w_{s+1}}]
&\leq
\frac{1}{1 - 4 \stepParameter}
\left[
\left(
\frac{\conditionNumber}{m \stepParameter} + 4 \stepParameter \right)
\E[\pError{\tilde w_s}]
\right.
\\
&
\left.
+
\frac{1}{k}
\left(
\left(
2 \selfConcordance \sigma \conditionNumber
+ 9 \conditionNumber
\right) \sqrt{ \E[\pError{\tilde w_s}]}
+
\left(
1 + 2\sqrt{\eta}
+ \frac{10 \selfConcordance \sigma \conditionNumber}{\sqrt{k}}
\right) \sigma
\right)^2
\right]
\end{align*}
\end{theorem}

The proof utilizes the following lemmas.  First, we show how self
concordance implies that there is a better effective strong convexity
parameter in $\nabla^2P(w_*)$ norm when we are close to $w_*$.

\begin{lemma}
\label{lemma:selfcorcordantbound}
If $P$ is $\selfConcordance$-self-concordant, then
\begin{equation}\label{eq:self-concordance}
P(w_*) \ge P(w_t) + (w_*-w_t)^\top \nabla P(w_t) +
\frac{\|w_t-w_*\|^2_{\nabla^2
P(w_*)}}{2(1+\selfConcordance\|w_t-w_*\|_{\nabla^2 P(w_*)})^2}.
\end{equation}
\end{lemma}

\begin{proof}
First we use the property of self-concordant functions: if $f$ is $\selfConcordance$-self-concordant, then
$$
f(t) \ge f(0) + tf'(0)+ \frac{4}{\selfConcordance^2}\left(t\frac{\selfConcordance}{2}\sqrt{f''(0)}-\ln\left(1+t\frac{\selfConcordance}{2}\sqrt{f''(0)}\right)\right).
$$

Apply this property to the function $P$ restricted to the line between
$w_t$ and $w_*$, where the $0$ point is at $w_t$ and $t$ is
$\|w_t-w_*\|_{\nabla^2 P(w_t)}$, then we have
\[
P(w_*) \ge P(w_t) + (w_*-w_t)^\top \nabla P(w_t) + \frac{4}{\selfConcordance^2}\left( \frac{\selfConcordance}{2}\|w_t-w_*\|_{\nabla^2 P(w_t)} - \ln\left(1+\frac{\selfConcordance}{2}\|w_t-w_*\|_{\nabla^2 P(w_t)}\right)\right).
\]

In order to convert $\nabla^2 P(w_t)$ norm to $\nabla^2 P(w_*)$ norm, we use another property of self-concordant function:
$$
f''(t) \ge \frac{f''(0)}{(1+t\frac{\selfConcordance}{2}\sqrt{f''(0)})^2}.
$$
Again we restrict to the line between $w_*$ and $w_t$, where $0$ point corresponds to $w_*$, and $t$ is $\|w_t-w_*\|$, and we get
$$
\|w_t-w_*\|_{\nabla^2 P(w_t)}^2 \ge \frac{\|w_t-w_*\|_{\nabla^2 P(w_*)}^2}{(1+\frac{\selfConcordance}{2}\|w_t-w_*\|_{\nabla^2 P(w_*)})^2}.
$$

Now consider the function let $h(x) = x-\ln(1+x)$. The function has the following two properties:
When $x\ge 0$, $h(x)$ is monotone and $h(x)\ge x^2/2(1+x)$. This claim
can be verified directly by taking derivatives.

Therefore
\begin{align*}
h\left(\frac{\selfConcordance}{2}\|w_t-w_*\|_{\nabla^2 P(w_t)}\right)
&
\ge h\left(\frac{\frac{\selfConcordance}{2}\|w_t-w_*\|_{\nabla^2 P(w_*)}}{(1+\frac{\selfConcordance}{2}\|w_t-w_*\|_{\nabla^2 P(w_*)})}\right)
\\
& \ge \frac{\frac{\selfConcordance^2}{4}\|w_t-w_*\|_{\nabla^2 P(w_*)}^2}{(1+\frac{\selfConcordance}{2}\|w_t-w_*\|_{\nabla^2 P(w_*)})^2}
\cdot
\frac{1}{2\left(1+\frac{\frac{\selfConcordance}{2}\|w_t-w_*\|_{\nabla^2 P(w_*)}}{(1+\frac{\selfConcordance}{2}\|w_t-w_*\|_{\nabla^2 P(w_*)})}\right)}
\\
& = \frac{\frac{\selfConcordance^2}{4}\|w_t-w_*\|_{\nabla^2 P(w_*)}^2}{2(1+\frac{\selfConcordance}{2}\|w_t-w_*\|_{\nabla^2 P(w_*)})(1+\selfConcordance\|w_t-w_*\|_{\nabla^2 P(w_*)})}\\
&
\ge \frac{\selfConcordance^2\|w_t-w_*\|_{\nabla^2 P(w_*)}^2}{8(1+\selfConcordance\|w_t-w_*\|_{\nabla^2 P(w_*)})^2} \, .
\end{align*}

This concludes the proof.
\end{proof}

Essentially, this means when $\|w_t-w_*\|_{\nabla^2P(w_*)}^2$ is small
the effective strong convexity in $\norm{\cdot}_{\hessOpt}$
is small. In particular,
\[
\alpha_t \leq
\min \left\{  \alpha , \left(1 + \frac{\selfConcordance}{2} \| w_t -
w_* \|_{\hessOpt}\right)^2 \right\}
\leq
\min \left\{  \conditionNumber , \left(1 + \frac{\selfConcordance}{2} \| w_t -
w_* \|_{\hessOpt}\right)^2 \right\}
\]

Thus we need to bound the residual error $ \|w_t - w_*\|_{\hessOpt}^2$.

\begin{lemma}[Crude residual error bound]
\label{lem:expected-residual}
Suppose the same assumptions in Lemma~\ref{lem:contraction_one_step}
hold and that $\stepParameter \leq \frac{1}{8}$. Then for all $t$,
we have
\[
\E \|w_t - w_*\|_{\hessOpt}^2
\leq
3 \conditionNumber
(\pError{\tilde w_0})
+ 6 \conditionNumber^2
\left\|\Delta\right\|^2 _{(\nabla^2  P(w_*))^{-1}}
\]
\end{lemma}

\begin{proof}
Since $\alpha_t \leq  \conditionNumber$  and by Lemma~\ref{lem:contraction_one_step} we have
\begin{align*}
\E L \|w_{t+1}-w_*\|^2
&\leq
\E \left[ L
\|w_t - w_*\|^2 -2 \stepParameter (1 - 4 \stepParameter) \left(P(w_t) -P(w_*)\right)
+ 8 \stepParameter^2 \left(P(\tilde{w}_0) -P(w_*)\right)
\right.
\\
&\enspace
\left.
+ \left(\conditionNumber \stepParameter + 2 \stepParameter^2 \right)
\left\|\Delta\right\|^2 _{(\nabla^2  P(w_*))^{-1}}
\right]
\end{align*}
Using that by strong convexity $P(w_t) - P(w_*) \geq \frac{\mu}{2} \|w_t - w_*\|^2$ we have
\begin{align*}
\E L\|w_{t+1}-w_*\|^2
&\leq
\E \left[
\left(1 - \frac{\stepParameter}{2 \conditionNumber}\right) L
\|w_t - w_*\|^2 \right]
+ \stepParameter \left(P(\tilde{w}_0) -P(w_*)\right)
+ 2 \stepParameter \conditionNumber
\left\|\Delta\right\|^2 _{(\nabla^2  P(w_*))^{-1}}
\end{align*}
Solving for the maximum value of $L \|w_t - w_*\|_2^2$ in this recurrence we have, for all $t$,
\begin{align*}
\E L \|w_{t} - w_*\|^2
&\leq
\frac{3\conditionNumber}{\stepParameter}
\left(
\stepParameter\left(P(\tilde{w}_0) -P(w_*)\right)
+ 2 \stepParameter \conditionNumber
\left\|\Delta\right\|^2 _{(\nabla^2  P(w_*))^{-1}}
\right)
\end{align*}
Using that $\nabla^2 P(w_*) \preceq  L \identityMatrix$ yields the result.
\end{proof}

Finally, we end up needing to bound higher moments of the error
from $\Delta$. For this we provide two technical lemmas.

\begin{lemma}
\label{lem:delta_bound}
Suppose Assumption~\ref{ass:main} and \ref{ass:selfconcordance} hold. For $\tilde{\psi}_i$ sampled independently, we have
\[
\E \left\|\frac{1}{k} \sum_{i \in {k}} \tilde{\psi}_i (w_*)
\right\|_{\hessOptInv}^4
\leq
12 \left(1 + \frac{\kurtosis }{k}\right) \left(\frac{\sigma^2}{k}\right)^2
\]
\end{lemma}

\begin{proof}
By Assumption~\ref{ass:selfconcordance} we have
\begin{align*}
&\E \left\|\frac{1}{k} \sum_{i \in {k}} \tilde{\psi}_i (w_*)
\right\|_{\hessOptInv}^4\\
=&
\frac{1}{k^4}
\left[
k \left( \E_{\psi \sim \D}  \| \nabla \psi(w_*) \|_{\hessOptInv}^4 \right)
+ 3k(k - 1)
\left(
\E_{\psi \sim \D}  \| \nabla \psi_i(w_*) \|^2_{\hessOptInv}
\right)^2
\right]
\\
\leq&
\frac{3k(k - 1) + \kurtosis k}{k^4}
\left(
\E_{\psi \sim \D}  \| \nabla \psi_i(w_*) \|^2_{\hessOptInv}
\right)^2
\end{align*}
Recalling the definition of $\sigma^2$ yields the result.
\end{proof}

\begin{lemma}
\label{lem:random-fourth-moment}
Suppose $a$ is a random variable such that $\E[a^4] \le \tilde C \cdot
(\E[a^2])^2$,
$b$ is a random variable, and $c$ is a constant. We have
\begin{align}
\E[a^2 \min\{b^2, c\}] \le 2 \E[a^2] \sqrt{\tilde C \cdot c \cdot \E[b^2]}.
\end{align}
\end{lemma}

\begin{proof}
Let $E_1$ be the indicator variable for the event $a^2 \ge T\E[a^2]$ where $T$ is chosen later. Let $E_2 = 1-E_1$.

On one hand, we have $\E[a^2 E_1] T\E[a^2] \le \E[a^4]$, therefore $\E[a^2 E_1] \le \frac{\tilde C}{T}\E[a^2]$.
On the other hand, $\E[\min\{b^2, c \} a^2 E_2] \le \E[b^2 a^2 E_2] \le T\E[a^2] \E[b^2]$.
Combining these two cases we have:
\begin{align*}
\E[\min\{b^2, c \} a^2] & = \E[\min\{b^2, c \} a^2 E_1] + \E[\min\{b^2, c \} a^2E_2] \\
& \le c \E[a^2 E_1] + \E[b^2a^2 E_2] \\
& \le \frac{c \cdot \tilde C}{T} \E[a^2] + T\E[a^2]\E[b^2] \\
& = 2\E[a^2]\sqrt{c \cdot \tilde C\E[b^2]}.
\end{align*}
In the last step we chose $T = \sqrt{\frac{c \cdot \tilde C}{\E[b^2]}}$ to balance the terms.
\end{proof}

Using these lemmas, we are ready to provide the proof.

\begin{proof}\textbf{of Theorem~\ref{thm:selfconcordanceconverge}.}
We analyze stage $s$ of the algorithm. Let us define the variance term (A) as
\begin{align*}
(A) =
\E\left[
\left(
\frac{\alpha_{\tilde m} + 2 \stepParameter}{2}
\right)
\| \Delta \|_{\hessOptInv}^{2}
\right]
\end{align*}
Our main goal in the proof is to bound $(A)$. First, for all $\alpha, x, y$ and positive semidefinite $H$ we have
\begin{align}
\E \alpha \|x + y\|_{H}^2
&= \E \left[\|
H^{-1/2} \sqrt{\alpha} x + H^{-1/2} \sqrt{\alpha} y \|_2^2
\right]
\nonumber
\\
&\leq
\left(
\sqrt{\E \|\sqrt{\alpha} H^{-1/2} x\|^2_{2}}
+
\sqrt{\E \|\sqrt{\alpha} H^{-1/2} x\|^2_{2}}
\right)^2
\nonumber
\\
& \leq
\left(
\sqrt{\E \alpha \|x\|^2_{H}}
+
\sqrt{\E \alpha \|y\|^2_{H}}
\right)^2 \, .
\label{eq:psd_fact}
\end{align}
By the definition of $\Delta$ we have
\[
(A)
\leq
\left(
\sqrt{(B)}
+
\sqrt{(\kurtosis )}
\right)^2
\]
where (B) and (C) are defined below. Using that $\E \|a -
\E[a]\|^2_{H} \leq \E \|a\|^2_{H}$, Lemma~\ref{lemma:smoothness},
and the strong convexity of $P$ we have
\begin{align*}
(B)
&=
\E \left(\frac{\alpha_{\tilde m} + 2 \stepParameter}{2}\right)
\left\|
\frac{1}{k} \sum_{i \in {k}} \nabla \tilde{\psi}_i (\tilde{w}_s) - \nabla \tilde{\psi}_i (w_*)
- \nabla P(\tilde{w}_s)
\right\|^2_{\hessOptInv}
\\
&\leq \left(\frac{\conditionNumber + 2 \stepParameter}{2}\right)
\cdot
\frac{2 \conditionNumber}{k} \cdot \E[\pError{\tilde{w}_s}]
\\
&\leq
\frac{2 \conditionNumber^2}{k} \cdot \E[\pError{\tilde{w}_s}] \, .
\end{align*}
We use that $\min\{a,b+c\} \leq b +\min\{a,c\}$ (for positive $a$, $b$, and $c$)
by Lemma~\ref{lemma:smoothness}, the definition of $\sigma^2$, as well as \eqref{eq:psd_fact}
\begin{align*}
(C)
&=
\E \left(\frac{\alpha_{\tilde m} + 2 \stepParameter}{2}\right)
\left\|\frac{1}{k} \sum_{i \in {k}} \nabla \tilde{\psi}_i (w_*)
\right\|_{\hessOptInv}^2
\\
&\leq
\frac{2 \stepParameter \sigma^2}{k}
+
\E
\left[
\frac{
\min \{\conditionNumber, (1+\selfConcordance\|w_t-w_*\|_{\nabla^2 P(w_*)})^2 \}
}{2}
\left\|\frac{1}{k} \sum_{i \in {k}} \nabla \tilde{\psi}_i (w_*)
\right\|_{\hessOptInv}^2
\right]
\\
&=
\frac{2 \stepParameter \sigma^2}{k}
+
\E
\left[
\left\|
\frac{\min \{\sqrt{\conditionNumber}, 1+\selfConcordance\|w_t-w_*\|_{\nabla^2 P(w_*)} \}}{\sqrt{2}k}
\sum_{i \in {k}} \nabla \tilde{\psi}_i (w_*)
\right\|_{\hessOptInv}^2
\right]
\\
&\leq
\frac{2 \stepParameter \sigma^2}{k}
+
\E
\left[
\left\|
\frac{1}{\sqrt{2}k}
\sum_{i \in {k}} \nabla \tilde{\psi}_i (w_*)
+
\frac{\min \{\sqrt{\conditionNumber}, \selfConcordance\|w_t-w_*\|_{\nabla^2 P(w_*)} \}}{\sqrt{2}k}
\sum_{i \in {k}} \nabla \tilde{\psi}_i (w_*)
\right\|_{\hessOptInv}^2
\right]
\\
&\leq
\frac{2 \eta \sigma^2}{k}
+
\left(
\sqrt{\frac{\sigma^2}{k}}
+
\sqrt{ (D)}
\right)^2
\end{align*}
where (D) is defined below. Using Lemma~\ref{lem:expected-residual}
and the independence of the different types of $\psi$
\begin{align*}
(D)
&=
\E
\left[
\frac{
\min \{\conditionNumber, \selfConcordance^2 \|w_t-w_*\|_{\nabla^2 P(w_*)}^2 \}
}{2}
\left\|\frac{1}{k} \sum_{i \in {k}} \nabla \tilde{\psi}_i (w_*)
\right\|_{\hessOptInv}^2
\right]
\\
&\leq
\E \left[
\frac{\min \left\{
\conditionNumber,
\selfConcordance^2
\left(3 \conditionNumber \cdot \pError{\tilde{w}_s}
+ 6 \conditionNumber^2 \|\Delta \|_{\hessOptInv}^{2}\right)
\right\}}{2}
\left\|\frac{1}{k} \sum_{i \in {k}} \nabla \tilde{\psi}_i (w_*)
\right\|_{\hessOptInv}^2
\right]
\\
&\leq
\frac{3 \conditionNumber^2 \selfConcordance^2 \sigma^2}{k} \E[\pError{\tilde{w}_s}]
+ \frac{\conditionNumber}{2} (E)
\end{align*}
where (E) is defined below. Using kurtosis,
\begin{align*}
\E \left[ \left\|\frac{1}{k} \sum_{i \in {k}} \tilde{\psi}_i (w_*) \right\|_{\hessOptInv}^4 \right] &\leq 14 (\sigma^2/k)^2.
\end{align*}
By
Lemma~\ref{lem:delta_bound} and applying
Lemma~\ref{lem:random-fourth-moment} we have
\begin{align*}
(E)
&\leq
\E \left[
\min \left\{
1, 6\conditionNumber \selfConcordance^2 \|\Delta \|_{\hessOptInv}^{2}
\right\}
\left\|\frac{1}{k} \sum_{i \in {k}} \nabla \tilde{\psi}_i (w_*)
\right\|_{\hessOptInv}^2
\right]
\\
&\leq
\E \left[
\min \left\{
1, 12\conditionNumber \selfConcordance^2
\left(
\left\|\frac{1}{k} \sum_{i \in {k}} \nabla \tilde{\psi}_i (\tilde{w}_s) - \nabla \tilde{\psi}_i (w_*)
-\nabla P(\tilde{w}_s)\right\|_{\hessOptInv}^2
+
\left\|\frac{1}{k} \sum_{i \in {k}} \nabla \tilde{\psi}_i (w_*)
\right\|_{\hessOptInv}^2\right)
\right\}
\right.
\\
& \enspace \enspace
\left.
\cdot
\left\|\frac{1}{k} \sum_{i \in {k}} \nabla \tilde{\psi}_i (w_*)
\right\|_{\hessOptInv}^2
\right]
\\
&\leq
\E \left[
\min \left\{
1, 12\conditionNumber \selfConcordance^2
\left\|\frac{1}{k} \sum_{i \in {k}} \nabla \tilde{\psi}_i (\tilde{w}_s) - \nabla \tilde{\psi}_i (w_*)
-\nabla P(\tilde{w}_s)\right\|_{\hessOptInv}^2
\right\}
\cdot
\left\|\frac{1}{k} \sum_{i \in {k}} \nabla \tilde{\psi}_i (w_*)
\right\|_{\hessOptInv}^2
\right]
\\
& \enspace \enspace
+ 170 \conditionNumber \selfConcordance^2 \left(\frac{\sigma^2}{k}\right)^2
\\
&\leq
\frac{4 \sigma^2}{k}
\sqrt{14 \cdot 1 \cdot 24 \conditionNumber^2 \selfConcordance^2\frac{\pError{\tilde{w}_s}}{k}}
+ 170 \conditionNumber \selfConcordance^2 \left(\frac{\sigma^2}{k}\right)^2
\\
&\leq
2 \frac{4 \sqrt{\conditionNumber} \selfConcordance \sigma^2}{k}
\sqrt{96 \conditionNumber \frac{\pError{\tilde{w}_s}}{k}}
+ 170 \conditionNumber \selfConcordance^2
\left(\frac{\sigma^2}{k}\right)^2 \quad \quad \quad \textrm{by manipulation
of constants}
\\
&\leq 16 \conditionNumber \selfConcordance^2 \left(\frac{\sigma^2}{k}\right)^2
+ \frac{96 \conditionNumber [\pError{\tilde{w}_s}]}{ k}
+170 \conditionNumber\selfConcordance^2
\left(\frac{\sigma^2}{k}\right)^2 \quad \quad \textrm{since } 2a \cdot b
\leq a^2 +b^2
\\
&\leq
200 \selfConcordance^2 \conditionNumber \left(\frac{\sigma^2}{k}\right)^2
+ \frac{96 \conditionNumber [\pError{\tilde{w}_s}]}{ k}
\end{align*}
Using that $\sqrt{|x| + |y|} \leq \sqrt{|x|} + \sqrt{|y|}$ this implies
\begin{align*}
(A)
&\leq
\left(
\sqrt{\frac{2 \conditionNumber^2}{k} \cdot \E[\pError{\tilde{w}_s}]}
+
\sqrt{\frac{2 \eta \sigma^2}{k}
+
\left(
\sqrt{\frac{\sigma^2}{k}}
+
\sqrt{
(D)
}
\right)^2
}
\right)^2
\\
&\leq
\left(
\frac{2 \conditionNumber}{\sqrt{k}} \sqrt{ \E[\pError{\tilde{w}_s}]}
+
\frac{2 \sigma \sqrt{\eta}}{\sqrt{k}}
+
\frac{\sigma}{\sqrt{k}}
+
\sqrt{ (D) }
\right)^2
\\
&\leq
\left(
\frac{2 \conditionNumber}{\sqrt{k}} \sqrt{ \E[\pError{\tilde{w}_s}]}
+
\frac{2 \sigma \sqrt{\eta}}{\sqrt{k}}
+
\frac{\sigma}{\sqrt{k}} \right.
\\
& \left.
+
\sqrt{ \frac{3 \conditionNumber^2 \selfConcordance^2 \sigma^2}{k} \E[\pError{\tilde{w}_s}]
+ \frac{\conditionNumber}{2}\left( 200 \selfConcordance^2 \conditionNumber \left(\frac{\sigma^2}{k}\right)^2
+ \frac{96 \conditionNumber [\pError{\tilde{w}_s}]}{ k} \right) }
\right)^2
\\
&\leq
\frac{1}{k}
\left(
\left(
2 \conditionNumber + 2 \selfConcordance \sigma \conditionNumber
+ 7 \conditionNumber
\right) \sqrt{ \E[\pError{\tilde{w}_s}]}
+
\left(
1 + 2\sqrt{\eta}
+ \frac{10 \selfConcordance \sigma \conditionNumber}{\sqrt{k}}
\right) \sigma
\right)^2
\end{align*}
Using this bound in Lemma~\ref{lem:contraction_multiple_steps} then
yields the result.
\end{proof}

\section{Empirical risk minimization ($M$-estimation) for smooth functions}
\label{sec:erm}

We now provide finite-sample rates for the ERM. We take the domain
$\mathcal{S}$ to be compact in \eqref{eq:P} (see
Remark~\ref{remark:compact}).  Throughout this section, define:
\[
\| A \|_* = \|(\nabla^2 P(w_*))^{-1/2} \cdot A \cdot (\nabla^2 P(w_*))^{-1/2}\|
\]
for a matrix $A$ (of appropriate dimensions).

\begin{theorem} \label{theorem:ERM}
Suppose $\psi_1,\psi_2,\ldots $ are an independently drawn sample from
$\mathcal{D}$. Assume:
\begin{enumerate}
\item (Convexity of $\psi$) Assume that $\psi$ is convex (with
probability one).
\item (Smoothness of $\psi$) Assume that $\psi$
is smooth in the following sense: the
first, second, and third derivatives exist at all interior points of
$\mathcal{S}$ (with probability one).
\item (Regularity Conditions)  Suppose
\begin{enumerate}
\item $\mathcal{S}$ is compact (so $P(w)$ is bounded on $\mathcal{S}$).
\item $w_*$ is an interior point of $\mathcal{S}$.
\item $\nabla^2 P(w_*)$ is positive definite (and, thus, is invertible).
\item There exists a neighborhood $B$ of $w_*$ and a constant $L_3$,
such that (with probability one) $\nabla^2\psi(w)$ is
$L_3$-Lipschitz, namely $\|\nabla^2\psi(w)-\nabla^2\psi(w')
\|_{*}\leq L_3 \|w-w'\|_{\nabla^2 P(w_*)}$, for $w,w'$ in this
neighborhood.
\end{enumerate}
\item (Concentration at $w_*$) Suppose $\|\nabla\psi(w_*)\|_{(\nabla^2 P(w_*))^{-1}}\leq
L_1$ and $\|\nabla^2 \psi(w_*)\|_*\leq L_2$ hold with probability
one. Suppose the dimension $d$ is finite (or, in the infinite
dimensional setting, the intrinsic dimension is bounded, as in
Remark~\ref{remark:dimension}).
\end{enumerate}

Then:
\begin{eqnarray*}
\lim_{N \rightarrow \infty}\frac{\E[ P(\widehat w^{\textrm{ERM}}_{N}) - P(w_*)]} {\sigma^2/N}
& = & 1
\end{eqnarray*}
In particular, the following lower and upper bounds hold. Define
\begin{eqnarray*}
\eps_N := c\left( L_1 L_3 + \sqrt{L_2} \right) \sqrt{\frac{p\log  d N}{N}}
\end{eqnarray*}
where $c$ is an appropriately chosen universal constant. Also, let
$c'$ be another appropriately chosen universal constant.  We have that
for all $p\geq 2$, if $N$ is large enough so that $\sqrt{\frac{p\log d
N}{N}}\leq c'
\min\left\{\frac{1}{\sqrt{L_2}},\frac{1}{L_1L_3},\frac{1 \cdot
\textrm{diameter}(B)}{L_1}\right\}$, then
\begin{eqnarray*}
\left(1-\eps_N\right) \frac{\sigma^2}{N}
- \frac{\sqrt{\E[Z^4]}}{N^{p/2}}
&\leq&
\E[ P(\widehat w^{\textrm{ERM}}_{N}) - P(w_*)]\\
&\leq&
\left(1+\eps_N\right) \frac{\sigma^2}{N} + \frac{\max_{w\in
\mathcal{S}} \left(P(w)-P(w_*)\right)}{N^p}
\end{eqnarray*}
where $Z=\left\|\nabla \widehat
P_N(w_*)\right\|_{(\nabla^2 P(w_*))^{-1}}$ and so
$\sqrt{\E[Z^4]}\leq L_1^2$.
The lower bound above holds even if $\mathcal{S}$ is not compact.
\end{theorem}

\begin{remark}[Infinite dimensional setting]
\label{remark:dimension}
Define $M= \nabla^2 \psi(w_*)-\nabla^2 P(w_*)$ and $\tilde d =
\frac{\tr(\E M^2)}{\lambda_{\max}(\E M^2)}$, which we assume to be
finite. Here can replace $d$ with $\tilde d$ in the theorems. See
Lemma~\ref{lemma:matrix-bernstein}.
\end{remark}

\begin{remark}[Compactness of $\mathcal{S}$]
\label{remark:compact}
The lower bound holds even if $\mathcal{S}$ is not compact. For the
upper bound, the proof technique uses the compactness of
$\mathcal{S}$ to bound the contribution to the expected regret due
to a (low probability) failure event that the ERM may not lie in the
ball $B$ (or even the interior of $\mathcal{S}$).
If $P$ is regularized then this last term can be improved, as
$\mathcal{S}$ need not be compact.
\end{remark}

The basic idea of the proof follows that of~\citet{HKZ_regression},
along with various arguments based on Taylor's theorem.

\begin{proof}
Throughout the proof use $\widehat w_N$  to denote the ERM $\widehat w^{\textrm{ERM}}_{N}$.
Define:
\[
\widehat P_N(w) = \frac{1}{N} \sum_{i=1}^N \psi_i(w)
\]
which is convex as it is the average of convex functions.

Throughout the proof we take $t=c p \log ( d N)$ in the tail probability
bounds in Appendix~\ref{appendix:tail} (for some universal constant $c$). This implies a
probability of error less than $\frac{1}{N^p}$.

For all $w \in B$, the empirical function $\nabla^2 \widehat P_N(w) $
is $L_3$-Lipschitz. In Lemma~\ref{lemma:matrix-bernstein} in
Appendix~\ref{appendix:tail}, we may take $v \leq 2 \sqrt{L_2}$ (as all
eigenvalues of of $\nabla^2 P(w_*)$ are one, under the
choice of norm).  Using Lemma~\ref{lemma:matrix-bernstein} in
Appendix~\ref{appendix:tail}, for $w \in B$, we have:
\begin{align}
\|\nabla^2 \widehat P_N(w) - \nabla^2 P(w_*)\|_*
& \leq
\|\nabla^2 \widehat P_N(w)-\nabla^2 \widehat P_N(w_*)\|_*
+\|\nabla^2 \widehat P_N(w_*)-\nabla^2 P(w_*)\|_* \nonumber\\
& \leq
L_3 \|w-w_*\|_{\nabla^2  P(w_*)}
+c \sqrt{\frac{L_2 p \log  d N}{N}} \label{eq:local_lipschitz}
\end{align}
for some (other) universal constant $c$. Now we seek to ensure that $\widehat
P_N(w)$ is a constant spectral approximation to $\nabla^2 P(w_*)$.
By choosing a sufficiently smaller ball $B_1$ (choose $B_1$ to have
radius of $\min\{1/(10 L_3),\textrm{diameter}(B)\}$), the first term can be made
small for $w\in B_1$.
Also, for sufficiently large $N$, the second term can be made
arbitrarily small (smaller than $1/10$), which occurs if $\sqrt{\frac{p\log
d N}{N}}\leq \frac{c'}{\sqrt{L_2}}$.
Hence, for such large enough $N$, we have for $w \in B_1$:
\begin{equation}\label{eq:ball_hessian}
\frac{1}{2} \nabla^2 \widehat P_N(w) \preceq \nabla^2 P(w_*) \preceq 2 \nabla^2 \widehat P_N(w)
\end{equation}
Suppose $N$ is at least this large from now on.

Now let us show that $\widehat w_N \in B_1$, with high probability,
for $N$ sufficiently large.
By Taylor's theorem, for all $w$ in the interior of $\mathcal{S}$, there exists a $\tilde w$, between
$w_*$ and $w$, such that:
\[
\widehat P_N(w) = \widehat P_N(w_*)
+ \nabla \widehat P_N(w_*)^\top (w-w_*) + \frac{1}{2}(w-w_*)^\top  \nabla^2 \widehat P_N(\tilde w) (w-w_*)
\]
Hence, for all $w\in B_1$ and if Equation~\ref{eq:ball_hessian} holds,
\begin{eqnarray*}
\widehat P_N(w) - \widehat P_N(w_*) &=& \nabla \widehat P_N(w_*)^\top (w-w_*)
+\frac{1}{2}\|w-w_* \|^2_{\nabla^2 P(\tilde w)}\\
&\geq& \nabla \widehat P_N(w_*)^\top (w-w_*)
+\frac{1}{4}\|w-w_* \|^2_{\nabla^2 P(w_*)}\\
&\geq&
\|w-w_*\|_{\nabla^2 P(w_*)} \left(-\|\nabla \widehat P_N(w_*)\|_{(\nabla^2
P(w_*))^{-1}} + \frac{1}{4}\|w-w_* \|_{\nabla^2 P(w_*)}\right)
\end{eqnarray*}
Observe that if the right hand side is positive for some $w\in B_1$,
then $w$ is not a local minimum. Also, since $\|\nabla \widehat
P_N(w_*)\| \goesto{} 0$, for a sufficiently small value of $\|\nabla
\widehat P_N(w_*)\| $, all points on the boundary of $B_1$ will have
values greater than that of $w_*$. Hence, we must have a local minimum
of $\widehat P_N(w)$ that is strictly inside $B_1$ (for $N$ large
enough). We can ensure this local minimum condition is achieved by
choosing an $N$ large enough so that $\sqrt{\frac{p\log N}{N}}\leq c'
\min\left\{\frac{1}{L_1L_3},\frac{\textrm{diameter}(B)}{L_1}\right\}$,
using Lemma~\ref{lemma:vector-bernstein} (and our bound on the
diameter of $B_1$).  By convexity, we have that this is the global
minimum, $\widehat w_N$, and so $\widehat w_N \in B_1$ for $N$ large
enough. Assume now that $N$ is this large from here on.

For the ERM, $0 = \nabla \widehat P_N(\widehat w_N) $. Again, by
Taylor's theorem if $\widehat w_N$ is an interior point, we have:
\[
0 = \nabla \widehat P_N(\widehat w_N) = \nabla \widehat P_N(w_*) +  \nabla^2 \widehat P_N(\tilde w_N) (\widehat w_N-w_*)
\]
for some $\tilde w_N$ between $w_*$ and $\widehat w_N$.
Now observe that $\tilde w_N$ is in $B_1$ (since, for $N$ large enough,
$\widehat w_N \in B_1$).
Thus,
\begin{equation}\label{eq:ERM_definition_taylor}
\widehat w_N -w_* = (\nabla^2 \widehat P_N(\tilde w_N))^{-1} \nabla \widehat P_N(w_*)
\end{equation}
where the invertibility is guaranteed by
Equation~\ref{eq:ball_hessian} and the positive definiteness of $\nabla  P(w_*)$.
Using Lemma~\ref{lemma:vector-bernstein} in Appendix~\ref{appendix:tail},
\begin{align}\label{eq:ball_rate}
& \|\widehat w_N -w_*\|_{\nabla^2 P(w_*)} \nonumber \\
\leq & \|(\nabla^2  P(w_*))^{1/2} (\nabla^2 \widehat P_N(\tilde w_N))^{-1} (\nabla^2  P(w_*))^{1/2}\|
\|\nabla \widehat P_N(w_*)\|_{(\nabla^2  P(w_*))^{-1}}
\leq
c L_1 \sqrt{\frac{p \log d N}{N}}
\end{align}
for some universal constant $c$.

Again, by Taylor's theorem, we have that for some $\tilde z_N $:
\[
P(\widehat w_N) - P(w_*) = \frac{1}{2} (\widehat w_N-w_*)^\top  \nabla^2 P(\tilde z_N) (\widehat w_N-w_*)
\]
where $\tilde z_N $ is between $w_*$ and $\widehat w_N$.

Observe that both $\tilde w_N$ and $\tilde z_N$ are between $\widehat
w_N$ and $w_*$ , which implies $\tilde w_N \rightarrow w_*$ and $\tilde z_N \rightarrow
w_*$ since $\widehat w_N \rightarrow w_* $. By
Equations~\ref{eq:local_lipschitz} and ~\ref{eq:ball_rate} (and the
tail inequalities in Appendix~\ref{appendix:tail}),
\begin{align*}
\|\nabla^2 \widehat P_N(\tilde w_N) - \nabla^2 P(w_*) \|_* & \leq
c\left(L_1 L_3 + \sqrt{L_2} \right) \sqrt{\frac{p \log d N}{N}} \\
\|\nabla^2 P_N(\tilde z_N) - \nabla^2 P(w_*) \|_* & \leq
L_3 \|\tilde z_N - w_* \|_{\nabla^2 P(w_*)}
\leq
c L_1 L_3 \sqrt{\frac{p \log d N}{N}}
\end{align*}

Define:
\[
\eps_N = c\left(L_1 L_3 + \sqrt{L_2} \right) \sqrt{\frac{p \log d N}{N}}
\]
Here the universal constant $c$ is chosen so that:
\[
\left(1-\eps_N\right) \nabla^2 P(w_*)
\preceq
\nabla^2 P(\tilde z_N)
\preceq
\left(1+\eps_N\right) \nabla^2 P(w_*)
\]
and
\[
\left(1-\eps_N\right) \nabla^2 P(w_*)
\preceq
\nabla^2 \widehat P_N(\tilde w_N)
\preceq
\left(1+\eps_N\right) \nabla^2 P(w_*)
\]
(using standard
matrix perturbation results).

Define:
\begin{eqnarray*}
M_{1,N} &=& (\nabla^2 P(w_*))^{1/2}
(\nabla^2 \widehat P_N(\tilde w_N))^{-1} (\nabla^2 P(w_*))^{1/2} \\
M_{2,N} &=& (\nabla^2 P(w_*))^{-1/2} \nabla^2 P(\tilde z_N) (\nabla^2 P(w_*))^{-1/2}\,
\end{eqnarray*}
For a lower bound, observe that:
\begin{eqnarray*}
P(\widehat w_N) - P(w_*) & \geq & \frac{1}{2} \lambda_{\min}(M_{2,N})
\left\|\widehat w_N-w_*\right\|^2_{\nabla^2 P(w_*)} \\
& = & \frac{1}{2} \lambda_{\min}(M_{2,N})
\left\|\nabla^2 \widehat P_N(\tilde w_N)(\widehat w_N-w_*)\right\|^2_{(\nabla^2 \widehat P_N(\tilde w_N))^{-1}\nabla^2 P(w_*) (\nabla^2 \widehat P_N(\tilde w_N))^{-1}} \\
& \geq & \frac{1}{2} (\lambda_{\min}(M_{1,N}))^2 \lambda_{\min}(M_{2,N})
\left\|\nabla^2 \widehat P_N(\tilde w_N)(\widehat w_N-w_*)\right\|^2_{(\nabla^2 P(w_*))^{-1}}\\
& = & \frac{1}{2} (\lambda_{\min}(M_{1,N}))^2 \lambda_{\min}(M_{2,N})
\left\|\nabla \widehat P_N(w_*)\right\|^2_{(\nabla^2 P(w_*))^{-1}}\\
\end{eqnarray*}
where we have used the ERM expression in Equation~\ref{eq:ERM_definition_taylor}.

Let $I(\mathcal{E})$ be the indicator that the desired previous events
hold, which we can ensure with probability greater than
$1-\frac{c}{N^p}$.  We have:
\begin{eqnarray*}
&&\E[ P(\widehat w_N) - P(w_*)]\\
&\geq&\E[ (P(\widehat w_N) - P(w_*)) I(\mathcal{E})]\\
& \geq &
\frac{1}{2}\E\left[
(\lambda_{\min}(M_{1,N}))^2 \lambda_{\min}(M_{2,N})
\left\|\nabla \widehat P_N(w_*)\right\|^2_{(\nabla^2 P(w_*))^{-1}} I(\mathcal{E}) \right]\\
& \geq &
(1-c' \eps_N) \frac{1}{2}
\E\left[ \left\|\nabla \widehat P_N(w_*)\right\|^2_{(\nabla^2 P(w_*))^{-1}}
I(\mathcal{E}) \right]\\
& = &
(1-c' \eps_N) \frac{1}{2}
\E\left[ \left\|\nabla \widehat P_N(w_*)\right\|^2_{(\nabla^2 P(w_*))^{-1}}
\left(1-I(\textrm{not }\mathcal{E}) \right) \right]\\
&= &
(1-c' \eps_N) \left( \sigma^2- \frac{1}{2}
\E\left[ \left\|\nabla \widehat P_N(w_*)\right\|^2_{(\nabla^2 P(w_*))^{-1}}
I(\textrm{not }\mathcal{E}) \right]\right)\\
&\geq &
(1-c '\eps_N) \sigma^2-
\E\left[ \left\|\nabla \widehat P_N(w_*)\right\|^2_{(\nabla^2 P(w_*))^{-1}}
I(\textrm{not }\mathcal{E}) \right]\\
\end{eqnarray*}
(for a universal constant $c'$).
Now define the random variable $Z=\left\|\nabla \widehat
P_N(w_*)\right\|_{(\nabla^2 P(w_*))^{-1}}$.  With a failure event
probability of less than $\frac{1}{2 N^p}$, for any $z_0$, we have:
\begin{align*}
\E\left[ Z^2 I(\textrm{not }\mathcal{E}) \right] & =
\E\left[Z^2 I(\textrm{not }\mathcal{E}) I (Z^2\leq z_0) \right] +
\E\left[Z^2 I(\textrm{not }\mathcal{E}) I (Z^2\geq z_0) \right] \\
& \leq z_0 \E\left[I(\textrm{not }\mathcal{E}) \right] +
\E\left[Z^2 I (Z^2\geq z_0) \right] \\
& \leq \frac{z_0}{2N^p}+
\E\left[Z^2\frac{Z^2}{z_0} \right] \\
& \leq \frac{z_0}{2N^p}+
\frac{\E[Z^4]}{ z_0} \\
& \leq  \frac{\sqrt{\E[Z^4]}}{ N^{p/2}}
\end{align*}
where we have chosen $z_0 = N^{p/2}\sqrt{\E[Z^4]}$.

For an upper bound:
\begin{eqnarray}
\E[ P(\widehat w_N) - P(w_*)]
& = &
\E[ (P(\widehat w_N) - P(w_*)) I(\mathcal{E})]
+ \E[ (P(\widehat w_N) - P(w_*)) I(\textrm{not }\mathcal{E})]
\label{eq:erm_upper}\\
& \leq &
\E[ (P(\widehat w_N) - P(w_*)) I(\mathcal{E})] +
\frac{\max_{w\in
\mathcal{S}} \left(P(w)-P(w_*)\right)}{N^p} \nonumber
\end{eqnarray}
since the probability of $\textrm{not }\mathcal{E}$ is less than $\frac{1}{N^p}$.

For an upper bound of the first term, observe that:
\begin{eqnarray*}
&&\E[ (P(\widehat w_N) - P(w_*)) I(\mathcal{E})]\\
& \leq &
\frac{1}{2}\E\left[
(\lambda_{\max}(M_{1,N}))^2 \lambda_{\max}(M_{2,N})
\left\|\nabla \widehat P_N(w_*)\right\|^2_{\nabla^2 P(w_*)} I(\mathcal{E}) \right]\\
& \leq &
(1+c' \eps_N) \frac{1}{2}
\E\left[ \left\|\nabla \widehat P_N(w_*)\right\|^2_{(\nabla^2 P(w_*))^{-1}}
I(\mathcal{E}) \right]\\
& \leq &
(1+c' \eps_N) \frac{1}{2}\E\left[
\left\|\nabla \widehat P_N(w_*)\right\|^2_{(\nabla^2 P(w_*))^{-1}} \right]\\
& =  &
(1+c' \eps_N) \frac{\sigma^2}{N}
\end{eqnarray*}
This completes the proof (using a different universal constant $c'$ in $\eps_N$).
\end{proof}

\section*{Acknowledgments}

The authors would like to thank Jonathan Kelner, Yin Tat Lee, and Boaz
Barak for helpful discussion. Part of this work was done while RF and
AS were at Microsoft Research, New England, and another part done
while AS was visiting the Simons Institute for the Theory of
Computing, UC Berkeley. This work was partially supported by NSF
awards 0843915 and 1111109, NSF Graduate Research Fellowship (grant
no.\ 1122374).

\bibliographystyle{plainnat}
\bibliography{opt}

\appendix

\section{A weaker smoothness assumption}
\label{sec:weaker}

Instead, of the smoothness Assumption~\ref{ass:main} in
Equation~\ref{equation:smoothness}, we could instead directly assume
that:
\begin{align} \label{equation:av_smoothness}
\bigE {\psi \sim \D} { \| \nabla \psi(w) - \nabla \psi(w_*) \|^2 }
&\leq 2 L \left(P(w)-P(w_*)\right).
\end{align}
Our proofs only use this condition, as well as an upper bound on the Hessian of $P$ at $w_*$. However, we can show that this weaker assumption implies such an upper bound as follows.

\begin{lemma}\label{lem:hessian_opt_bound}
If \eqref{equation:av_smoothness} holds then $  \hess P(w_*) \preceq 2 L \identityMatrix $.
\end{lemma}

\begin{proof}
First we note that for all $w$, by \eqref{equation:av_smoothness}, the convexity of $P$, and Jensen's inequality, we have:
\begin{equation}
\label{eq:hess_opt_bound:1}
\|\nabla P(w) \|^2 =
\|\nabla P(w) - \nabla P(w_*) \|^2
\leq \E_{\psi \sim D} \| \nabla \psi(w) - \nabla \psi(w_*) \|^2
\leq 2 L (P(w) - P(w_*))
\, .
\end{equation}
Since $P$ is convex we also know that
\begin{equation}
\label{eq:hess_opt_bound:2}
P(w_*) \geq P(w) + \nabla P(w)^\top (w_* - w)
\, .
\end{equation}
Combining \eqref{eq:hess_opt_bound:1} and \eqref{eq:hess_opt_bound:2}  and using Cauchy-Schwarz yields that for all $w$,
\[
\|\nabla P(w) \|^2\leq 2L ( \nabla P(w)^\top (w_* - w) )
\leq 2L \|\nabla P(w) \| \| w_* - w \|
\]
Consequently, for all $w$ we have that
\begin{equation}
\label{eq:hess_opt_bound:3}
\|\nabla P(w) \| \leq 2L \| w_* - w \|,
\end{equation}

Now fix any $\Delta \in \domain$ and let $g : \R \rightarrow \R$ be defined for all $t \in \R$ as
\[
g_\Delta(t) \defeq f(w_* + t \Delta).
\]
By the chain rule we know that
\[
g'(t) = \Delta^\top \nabla P(w_* + t \Delta)
\enspace \text{ and } \enspace
g''(t) = \Delta^\top \nabla^2 P(w_* + t \Delta) \Delta.
\]
Consequently, by definition and the fact that $P$ is twice differentiable at $w_*$ we have
\[
\Delta^\top \nabla^2 P(w_*) \Delta
=
g''(0)
= \lim_{t \rightarrow 0}
\frac{\Delta^\top \nabla P(w_* + t \Delta)
- \Delta^\top \nabla P(w_*)
}
{t}
\enspace.
\]
Applying Cauchy-Schwarz and \eqref{eq:hess_opt_bound:3} yields that
\[
\Delta^\top \nabla^2 P(w_*) \Delta
\leq \lim_{t \rightarrow 0}
\frac{\| \Delta \| \cdot \|\nabla P(w_* + t \Delta) \|}{|t|}
\leq \lim_{t \rightarrow 0} \frac{\| \Delta \| \cdot 2L |t| \| \Delta \|}{|t|}
\leq 2L \| \Delta \|^2
\enspace.
\]
Since $\Delta$ was arbitrary we have the desired result.
\end{proof}

\section{Proofs of Corollaries~\ref{corollary:main} and~\ref{corollary:selfconcordant}}
\label{appendix:instantiations}

Throughout, define:
\[
E_s = \sqrt{\E[\pError{\hat{w}_{N_s}}]}
\]

\begin{proof}\textbf{of Corollary~\ref{corollary:main}.}
From Theorem~\ref{thm:alphaconverge}, we have:
\begin{equation*}
\sqrt{\E[\pError{\tilde w_{s + 1}}]} \leq
\frac{1}{\sqrt{1 - 4\eta }}
\left[\left(\sqrt{\frac{\conditionNumber}{m \eta} + 4\eta } +\sqrt{\conditionNumber\frac{\alpha + 2 \eta }{k}}\right)\sqrt{\E[\pError{\tilde w_s}]}
+ \sqrt{\alpha + 4 \eta }
\frac{\sigma}{\sqrt{k}}\right]
\end{equation*}

Let us first show that:
\begin{eqnarray*}
E_s & \leq \frac{E_{s-1}}{(\sqrt{b})^{p+1}}+
\left(1+\frac{1}{(\sqrt{b})^{(p+1)}}\right) \frac{\alpha\sigma}{\sqrt{ k_{s-1}}}
\end{eqnarray*}
We do this using Theorem~\ref{thm:alphaconverge} and some explicit
calculations as follows.
We shall make use of that for $x\leq 1$,  $\sqrt{1-x} \geq 1-x$, and for $0\leq x\leq
\frac{1}{2}$, $\frac{1}{1-x}\leq 1+x+2x^2 \leq 1+2x$. We have:

\begin{eqnarray*}
\frac{1}{\sqrt{1-4\eta}} & \leq & \frac{1}{\sqrt{1-\frac{1}{45}}} \leq 1.03
\\
\frac{1}{\sqrt{1-4\eta}} & \leq & 1-8\eta \leq 1+\frac{1}{2 b^{p+1}}
\\
\frac{\sqrt{\alpha+2\eta}}{\sqrt{1-4\eta}}
& \leq & \alpha (1+\sqrt{2\eta}) (1+\frac{1}{2 (b^{p+1})} )
\leq \alpha (1+\frac{1}{3 (\sqrt{b})^{p+1}}) (1+\frac{1}{2 (b^{p+1})} )
\leq \alpha \left(1+\frac{1}{(\sqrt{b})^{p+1}}\right)\\
\frac{\sqrt{4\eta+\frac{\conditionNumber}{\eta m}} }
{\sqrt{1-4\eta}}
& \leq & 1.03 \left(\sqrt{\frac{4}{20 (b^{p+1})}+\frac{1}{20 (b^{p+1})} )}\right)
\leq \frac{0.55}{(\sqrt{b})^{p+1}}\\
\frac{\sqrt{\conditionNumber\frac{\alpha + 2 \eta }{k}} }
{\sqrt{1-4\eta}}
& \leq & 1.03 \sqrt{\conditionNumber\frac{\alpha(1+ 1/40)}{k}}
\leq \frac{0.1}{(\sqrt{b})^{p+1}}\\
\frac{\sqrt{\frac{\conditionNumber}{m \eta} + 4\eta } +\sqrt{\conditionNumber\frac{\alpha + 2 \eta }{k}} }
{\sqrt{1-4\eta}}
&\leq & \frac{1}{(\sqrt{b})^{p+1}}
\end{eqnarray*}
This completes the claim, by substitution into Theorem~\ref{thm:alphaconverge}.

We now show:
\begin{eqnarray} \label{eq:err_bound}
E_s & \leq \frac{E_0}{(\sqrt{b})^{(p+1)s}}+
\left(1+\frac{2}{(\sqrt{b})^p}\right) \frac{\alpha\sigma}{\sqrt{ k_{s-1}}}
\, .
\end{eqnarray}
We do this by induction. The claim is true for $s=1$. For the
inductive argument,
\begin{align*}
E_s & \leq \frac{E_{s-1}}{(\sqrt{b})^{p+1}}+
\left(1+\frac{1}{(\sqrt{b})^{(p+1)}}\right) \frac{\alpha \sigma}{\sqrt{ k_{s-1}}}\\
& \leq \frac{E_0}{(\sqrt{b})^{(p+1)s}} +
\frac{1}{(\sqrt{b})^{(p+1)}} \left(1+\frac{2}{(\sqrt{b})^p}\right) \frac{\alpha\sigma}{\sqrt{ k_{s-2}}}
+
\left(1+\frac{1}{(\sqrt{b})^{p+1}}\right) \frac{\alpha\sigma}{\sqrt{
k_{s-1}}}\\
& = \frac{E_0}{(\sqrt{b})^{(p+1)s}} +
\left(1+\frac{2}{(\sqrt{b})^p}\right) \frac{\alpha \sigma}{(\sqrt{b})^{p}\sqrt{ k_{s-1}}}
+
\left(1+\frac{1}{(\sqrt{b})^{p+1}}\right) \frac{\alpha \sigma}{\sqrt{
k_{s-1}}}\\
&\leq \frac{E_0}{(\sqrt{b})^{(p+1)s}} + \left(1+\frac{2}{(\sqrt{b})^p}\right) \frac{\alpha \sigma}{\sqrt{ k_{s-1}}}
\end{align*}
which completes the inductive argument.

We now relate $k_s$ and $\sqrt{b}^{(p+1)s}$ to the sample size $N_s$.
First, observe that $m$ is bounded as:
\[
m = 400 b^{(p+1)^2} \conditionNumber
\leq
20\base^{(p+1)^2+3} \conditionNumber
=
20\base^{p^2+2p+4} \conditionNumber
\]
Using  $s>p^2+6p$,
\[
N_s = \sum_{\tau=1}^s (m+k_\tau)
\leq
s 20\base^{p^2+2p+4} \conditionNumber
+\frac{k_{s-1}}{1-\frac{1}{b}}
=
\frac{sk_s\base^{p^2+p+3}}{\base^s}
+\frac{k_s}{1-\frac{1}{b}}
\leq
(1+\frac{1.1}{b}) k_s
\]
and so:
\begin{equation}\label{eq:err_bound_term2}
\alpha \left(1+\frac{2}{(\sqrt{b})^p}\right) \frac{\sigma}{\sqrt{
k_{s-1}}} \leq
\left(1+\frac{2}{(\sqrt{b})^p}\right) \sqrt{1+\frac{1.1}{b}}
\frac{\alpha \sigma}{\sqrt{N_s}} \leq
\left(1+\frac{4}{b}\right) \frac{\alpha\sigma}{\sqrt{N_s}}
\end{equation}
Also,
\[
(\base^s)^{p+1} =(\base^s)^p \base^s
= \left( k_s \frac{1}{20 \alpha  \base^{p+1}\conditionNumber } \right)^p \base^s
\geq
\left( N_s \frac{1}{20 (1+\frac{1.1}{b}) \alpha  \base^{p+1}\conditionNumber} \right)^p \base^s
\geq
\left( \frac{1}{\alpha \conditionNumber} N_s \right)^p \frac{\base^s}{(\base^{p+5})^p}
\geq
\left( \frac{1}{\alpha \conditionNumber} N_s \right)^p
\]
where we have used that, for $s>p^2+6p$,
$\frac{\base^s}{(\base^{p+5})^p} \geq 1$.
Hence,
\begin{equation}\label{eq:err_bound_term1}
\frac{1}{\base^{s(p+1)}} \leq \frac{1}{\left( \frac{N_s}{\alpha
\conditionNumber} \right)^p } \, .
\end{equation}
The proof is completed substituting
~\eqref{eq:err_bound_term2},~\eqref{eq:err_bound_term1} in ~\eqref{eq:err_bound}.
\end{proof}

\begin{proof}\textbf{ of Corollary~\ref{corollary:selfconcordant}.}
Under the choice of parameters and using
Theorem~\ref{thm:alphaconverge}, we have
$$
E_{s} \le \frac{E_{s-1}}{\sqrt{\base}^{p+1}} + \frac{\sigma}{\sqrt{k_{s-1}}}\sqrt{3\conditionNumber + \frac{1}{\base}}.
$$

On the other hand, suppose $t_0$ is the first time that $k_{t_0} \ge
(M\sigma + 1)^2 400\conditionNumber^2 \base^{2p+3} = (M\sigma+1)^2
k_0$. When $s \ge t_0$, we can use
Theorem~\ref{thm:selfconcordanceconverge}, and under the choice of
parameters we have:
$$
E_{s} \le \frac{E_{s-1}}{\sqrt{\base}^{p+1}} + \frac{\sigma}{\sqrt{k_{s-1}}}\left(1+\frac{1}{\base}\right).
$$

Now we shall prove by induction that
$$
E_s \le \frac{E_{0}}{\sqrt{\base}^{(p+1)s}} + \frac{\sigma}{\sqrt{k_{s-1}}} (1+2/\base)\sqrt{3\conditionNumber+1/\base} \cdot \min\{1,(\sqrt{\base}^{-p(s-t_0)})\} + \frac{\sigma}{\sqrt{k_{s-1}}} (1+2/\base).
$$

Here $E_0$ is the initial error. When $s = 1$
the statement is true. When $s \le t_0$ we use the first
recursion (from Theorem~\ref{thm:alphaconverge}), clearly

\begin{align*}
E_s & \le \frac{E_{s-1}}{\sqrt{\base}^{p+1}} + \frac{\sigma}{\sqrt{k_{s-1}}}\sqrt{3\conditionNumber + \frac{1}{\base}}\\
& \le \left(\frac{E_0}{\sqrt{\base}^{(p+1)(s-1)}} + \frac{\sigma}{\sqrt{k_{s-2}}} (1+2/\base)\sqrt{3\conditionNumber+1/\base} + \frac{\sigma}{\sqrt{k_{s-2}}} (1+2/\base)\right)\cdot \frac{1}{\sqrt{\base}^{p+1}} + \frac{\sigma}{\sqrt{k_{s-1}}}\sqrt{3\conditionNumber + \frac{1}{\base}} \\
& = \frac{E_0}{\sqrt{\base}^{(p+1)s}} + \frac{\sigma}{\sqrt{k_{s-1}}} \sqrt{3\conditionNumber+1/\base} \left(\frac{1+2/\base}{\sqrt{b}^p} + 1\right) + \frac{\sigma}{\sqrt{k_{s-1}}} (1+2/\base) \cdot \frac{1}{\sqrt{b}^p} \\
& \le \frac{E_0}{\sqrt{\base}^{(p+1)s}} + \frac{\sigma}{\sqrt{k_{s-1}}} (1+2/\base)\sqrt{3\conditionNumber+1/\base} + \frac{\sigma}{\sqrt{k_{s-1}}} (1+2/\base).
\end{align*}
Here the second step uses induction hypothesis, and the third step uses the fact that $k_{s-1}/k_{s-2} = \base$ and $(1+2/b)/\sqrt{b}^{p} + 1 \le 1+2/b$.

When $s > t_0$ we can use the second recursion (from Theorem~\ref{thm:selfconcordanceconverge}), now we have
\begin{eqnarray*}
E_{s} & \le &\frac{E_{s-1}}{\sqrt{\base}^{p+1}} + \frac{\sigma}{\sqrt{k_{s-1}}}\sqrt{1 + \frac{1}{\base}}\\
& \le & \left(\frac{E_0}{\sqrt{\base}^{(p+1)(s-1)}} +
\frac{\sigma}{\sqrt{k_{s-2}}}
(1+2/\base)\sqrt{3\conditionNumber+1/\base} \cdot
(\sqrt{\base}^{-p(s-t_0-1)}) + \frac{\sigma}{\sqrt{k_{s-2}}}
(1+2/\base)\right)\frac{1}{\sqrt{\base}^{p+1}}\\
&& +\frac{\sigma}{\sqrt{k_{s-1}}}\left(1 +
\frac{1}{\base}\right) \\
& = & \frac{E_0}{\sqrt{\base}^{(p+1)s}} +
\frac{\sigma}{\sqrt{k_{s-1}}}
(1+2/\base)\sqrt{3\conditionNumber+1/\base} \cdot
(\sqrt{\base}^{-p(s-t_0)})
\\
&& +\frac{\sigma}{\sqrt{k_{s-1}}}
\left( \frac{(1+2/\base)}{\sqrt{\base}^{p}} + 1 +  \frac{1}{\base}\right) \\
& \le & \frac{E_0}{\sqrt{\base}^{(p+1)s}} + \frac{\sigma}{\sqrt{k_{s-1}}} (1+2/\base)\sqrt{3\conditionNumber+1/\base} \cdot (\sqrt{\base}^{-p(s-t_0)}) + \frac{\sigma}{\sqrt{k_{s-1}}} (1+2/\base).
\end{eqnarray*}
Again, the second step uses induction hypothesis, and final step uses $k_{s-1}/k_{s-2} = \base$ and $(1+2/b)/\sqrt{b}^{p} + 1+1/b \le 1+2/b$. This concludes the induction.

Finally, we need to relate the values $k_s$, $\sqrt{b}^{(p+1)(s+1)}$, $\sqrt{b}^{(p+1)(s-t_0)}$ with $N_s$.

First, it is clear that $k_s \ge \base m$ for all $s$, therefore
$$
N_s \le (1+1/\base) \sum_{t=0}^{s-1} k_s \le k_{s-1} (1-1/\base)^{-1}(1+1/\base) \le k_{s-1}(1+3/\base) \le 2k_{s-1}.
$$

Therefore we can substitute $1/\sqrt{k_{s-1}}$ with $\sqrt{1+3/b}/\sqrt{N_s}$.
Also, we know $N_s/2k_0 \le b^{s-1}$, therefore $\sqrt{b}^{-(p+1)s} \le \sqrt{b}^{-(p+1)(s-1)} \le (N/2k_0)^{-(p+1)/2}$.

Finally, since $k$ increase by a factor of $\base$, we know $k_{t_0}$
is at most $\base (M\sigma+1)^2k_0$. Therefore
$\frac{N_s}{2(M\sigma+1)^2k_0} \le bN_s/2k_{t_0}\le k_s/k_{t_0} =
b^{s-t_0}$, which means $\sqrt{b}^{-p(s-t_0)} \le
\left(\frac{N_s}{2(M\sigma+1)^2k_0}\right)^{p/2}$.
\end{proof}

\section{Self-concordance for logistic regression}

The following straightforward lemma, to handle self-concordance for
logistic regression, is included for completeness (see \citet{bach2010self} for a more detailed
treatment for analyzing the self-concordance of logistic regression).

\begin{lemma}
\label{lemma:selfcorcordantbound_logistic}
(Self-Concordance for Logistic Regression) For the logistic regression
case (as defined in Section~\ref{sec:applications}), define $\selfConcordance=\alpha\E[\|X\|^3_{(\nabla^2 P(w_*))^{-1}}] $, then
\[
P(w_*) \ge P(w_t) + (w_*-w_t)^\top \nabla P(w_t) +
\frac{\|w_t-w_*\|^2_{\nabla^2
P(w_*)}}{2(1+\selfConcordance\|w_t-w_*\|_{\nabla^2 P(w_*)})^2} \, .
\]
\end{lemma}

\begin{proof}
For $\widetilde M=\E[\|X\|^3_{(\nabla^2 P(w_*))^{-1}}] $, first let us
show that:
\begin{equation}\label{eq:logistic_taylor}
P(w_*) \ge P(w_t) + (w_*-w_t)^\top \nabla P(w_t) +
\frac{1}{2}\|w_t-w_*\|^2_{\nabla^2
P(w_*)}\max \left\{\frac{1}{\alpha},1-\widetilde M\|w_t-w_*\|_{\nabla^2 P(w_*)} \right\}
\end{equation}
By Taylor's theorem,
\[
P(w_*) = P(w_t) + (w_*-w_t)^\top \nabla P(w_t) + \frac{1}{2}(w_t-w_*)^\top \nabla^2 P(z_1) (w_t-w_*)
\]
where $z_1$ is between $w_*$ and $w_t$. Again, by Taylor's theorem,
\[
\nabla^2 P(z_1) = \nabla^2 P(w_*) + \nabla^3 P(z_2) (z_1-w_*)
\]
where $z_2$ is between $w_*$ and $z_1$.

By taking derivatives, we have that:
\begin{align*}
& (w_t-w_*)^\top \nabla^2 P(z_1) (w_t-w_*) \\
= & (w_t-w_*)^\top
\nabla^2 P(w_*) (w_t-w_*)  \\
&+ \E[\Pr(Y|z_2,X)
(1-\Pr(Y|z_2,X))(1-2\Pr(Y|z_2,X)) ((w_t-w_*)^\top X)^2 (z_1-w_*)^\top X]\\
\geq &\|w_t-w_*\|^2_{\nabla^2 P(w_*)}
-\|w_t-w_*\|^2 _{\nabla^2 P(w_*)} \|z_1-w_*\| _{\nabla^2 P(w_*)} \E[\|X\|^3_{(\nabla^2
P(w_*))^{-1}}]\\
\geq & \|w_t-w_*\|^2_{\nabla^2 P(w_*)}
-\|w_t-w_*\|^3 _{\nabla^2 P(w_*)}
\E[\|X\|^3_{(\nabla^2 P(w_*))^{-1}}]\\
=& \|w_t-w_*\|^2_{\nabla^2 P(w_*)}\left(1-\widetilde M\|w_t-w_*\|_{\nabla^2 P(w_*)})\right)
\end{align*}
Using the definition of $\alpha$, shows that
Equation~\ref{eq:logistic_taylor} holds.

Now for $Z>0$, consider the quantity $\max\{\frac{1}{\alpha},1-Z\}$, and
observe that the $1-Z$ term achieves the max when $1-Z\geq \frac{1}{\alpha}$ or equivalently when $-1+\alpha
-\alpha Z \geq 0$. Hence,
\begin{eqnarray*}
\max\left\{\frac{1}{\alpha},1-Z\right\} & = &
\max\left\{\frac{1}{\alpha},\frac{(1-Z)(1+\alpha Z)}{1+\alpha Z}\right\}\\
& = &
\max\left\{\frac{1}{\alpha},\frac{1-Z+\alpha Z-\alpha Z^2}{1+\alpha  Z}\right\}\\
& = &
\max\left\{\frac{1}{\alpha},\frac{1+Z(-1+\alpha-\alpha Z)}{1+\alpha Z}\right\}\\
& \geq &
\max\left\{\frac{1}{\alpha},\frac{1}{1+\alpha Z}\right\}\\
& \geq &
\frac{1}{1+\alpha Z}\\
& \geq &
\frac{1}{(1+\alpha Z)^2}\\
\end{eqnarray*}
Using this completes the proof.
\end{proof}

\section{Probability tail inequalities}
\label{appendix:tail}

The following probability tail inequalities are used in our analysis.

The first tail inequality is for sums of bounded random vectors; it
is a standard application of Bernstein's inequality.
\begin{lemma}[Vector Bernstein bound; \eg see
\citet{HKZ_vector}]
\label{lemma:vector-bernstein}
Let $x_1,x_2,\dotsc,x_n$ be independent random vectors such that
\[
\sum_{i=1}^n \E[ \|x_i\|^2 ] \leq v
\quad \text{and} \quad
\|x_i\| \leq r
\]
for all $i=1,2,\dotsc,n$, almost surely.
Let $s := x_1 + x_2 + \dotsb + x_n$.
For all $t > 0$,
\[
\Pr\left[ \|s\| > \sqrt{v} (1 + \sqrt{8t}) + (4/3) r t
\right] \leq e^{-t}
\]
\end{lemma}

The next tail inequality concerns the spectral accuracy of an empirical
second moment matrix, where we do not assume the dimension is finite.
\begin{lemma}[Infinite Dimensional Matrix Bernstein bound; \citet{HKZ_matrix}]
\label{lemma:matrix-bernstein}
Let $X$ be a random matrix, and $r > 0$, $v > 0$, and $\tilde d >
0$ be such that, almost surely,
\begin{gather*}
\E[X] = 0 , \quad
\lambda_{\max}[X] \leq r , \quad
\lambda_{\max}[\E[X^2]] =  v , \quad
\tr(\E[X^2]) = v \tilde d
.
\end{gather*}
Define $\tilde d$ as the intrinsic dimension.
If $X_1,X_2,\dotsc,X_n$ are independent copies of $X$, then for any $t >
0$,
\begin{equation*}
\Pr\left[
\lambda_{\max}\left[ \frac1n \sum_{i=1}^n X_i \right]
> \sqrt{\frac{2vt}{n}} + \frac{rt}{3n}
\right]
\leq \tilde d t (e^t - t - 1)^{-1}
.
\end{equation*}
If $t \geq 2.6$, then $t (e^t - t - 1)^{-1} \leq e^{-t/2}$.
\end{lemma}

\end{document}